\RequirePackage{fix-cm}
\documentclass[twocolumn]{svjour3}          
\smartqed  

\usepackage{graphicx}
\usepackage{times}
\usepackage{helvet}
\usepackage{courier}
\usepackage{amsmath,epsfig}
\usepackage{graphicx}
\usepackage{bm}
\usepackage{float,caption,multirow}
\usepackage{makecell}
\usepackage{url,caption}
\usepackage{multirow}
\usepackage{makecell}
\usepackage{wrapfig}
\usepackage{amssymb}
\usepackage{epstopdf}
\usepackage{booktabs} 
\usepackage{epsfig}

\usepackage{multirow}
\usepackage{xcolor}
\usepackage{helvet,url}
\usepackage{courier}
\usepackage{epsfig}
\usepackage{picinpar}
\usepackage{graphicx}
\usepackage{epstopdf}
\usepackage{bm}
\usepackage{epsfig}
\usepackage{picinpar}
\usepackage{bm}
\usepackage{float,caption,multirow}
\usepackage{marvosym}
\usepackage{subcaption}
\usepackage{framed}
\usepackage{multirow}
\usepackage{booktabs}
\usepackage{hyperref}
\usepackage{algorithm}
\usepackage{algorithmic}
\usepackage{csquotes}
\usepackage[export]{adjustbox}

\newcommand{\defeq}{\mathrel{\mathop:}=}

\DeclareMathOperator*{\argmin}{arg\,min}

 \definecolor{changecolor}{rgb}{0,0,1}  

\begin{document}
\sloppy
\title{PEdger++: Practical Edge Detection via Assembling Cross Information}




\author{Yuanbin Fu         \and
        Liang Li      \and
        Xiaojie Guo               
}

\institute{\Letter\ Xiaojie Guo \at
             \email{xj.max.guo@gmail.com}
             \and
           Yuanbin Fu \at 
              \email{yuanbinfu@tju.edu.cn}            
           \and
           Liang Li \at
           \email{liangli@tju.edu.cn}
           \and
             The authors are from the College of Intelligence and Computing, Tianjin University, Tianjin 300350, China. 
             This work was supported by the National Natural Science Foundation of China under Grants No.62072327 and 62372251.
}

\date{Received: date / Accepted: date}

\maketitle
\abstract{
Edge detection serves as a critical foundation for numerous computer vision applications, including object detection, semantic segmentation, and image editing, by extracting essential structural cues that define object boundaries and salient edges. To be viable for broad deployment across devices with varying computational capacities, edge detectors shall balance high accuracy with low computational complexity. While deep learning has evidently improved accuracy, they often suffer from high computational costs, limiting their applicability on resource-constrained devices.  This paper addresses the challenge of achieving that balance: \textit{i.e.}, {how to efficiently capture discriminative features without relying on large-size and sophisticated models}. We propose PEdger++, a collaborative learning framework designed to reduce computational costs and model sizes while improving edge detection accuracy. The core principle of our PEdger++ is that cross-information derived from  heterogeneous  architectures, diverse training moments, and multiple parameter samplings, is beneficial to enhance learning from an ensemble perspective. Extensive  experimental results on the BSDS500, NYUD and Multicue datasets demonstrate the effectiveness of our approach, both quantitatively and qualitatively, showing clear improvements over existing methods.  We also provide multiple versions of the model with varying computational requirements, highlighting PEdger++'s adaptability with respect to different resource constraints. Codes are accessible at \url{https://github.com/ForawardStar/EdgeDetectionviaPEdgerPlus/}.}



\section{Introduction}\label{sec:introduction}

%
%
%
%

In the field of computer vision,  edges perform as pivotal cues that reflect the structural essence of natural images, forming the basis for many high-level tasks such as object detection \cite{faster-rcnn}, semantic segmentation \cite{JSENet}, and image editing \cite{editing}, to name just a few. Ideally, a practical edge detector should be versatile enough to operate efficiently across various devices, including those with limited computational resources, \textit{e.g.}, mobile phones and tablets. Consequently, achieving high accuracy across varying computational complexities is critical for delivering optimal user experiences.

However, previous literature struggles to harmonize the conflicting demands of accuracy and efficiency in edge detection. Early edge extractors rely on shallow information, like image gradients  \cite{sobel,Canny} and hand-crafted features  \cite{pb,gpu-ucm,se}, resulting in poor accuracy. With the emergence of deep-learning techniques,  a number of methods \cite{HED,BDCN_cvpr,RCF,EDTER,treasure,Diffusionedge,MuGE} have been developed to leverage high-level/deep features. While these approaches significantly improve accuracy, their computational complexity becomes prohibitively expensive, limiting their applicability in real-time scenarios and resource-constrained devices. More recently, PiDiNet \cite{PiDiNet_TPAMI} has significantly accelerated the process, but it still lags behind state-of-the-art methods in terms of accuracy.  Besides, a critical issue often overlooked by existing deep-learning methods is that their network parameters are deterministic: network parameters are fixed after training, lacking the ability to handle epistemic uncertainty (see Fig.~\ref{vis} for illustration). To be more specific, as noted by \cite{uncertainties,bayesian1,bayesian2,subnetwork,bayesian3,bayesian4}, this uncertainty reflects the model's limited knowledge about input distributions, particularly those underrepresented or absent in the training data. High epistemic uncertainty often manifests as inconsistent or unreliable predictions on ambiguous inputs, severely degrading the detection accuracy. As a consequence, achieving a practical detector necessitates addressing both computational complexity and epistemic/model uncertainty to ensure reliable high accuracy across diverse scenarios and devices.

    

\begin{figure}[t]

    \includegraphics[width=0.24\linewidth, frame]{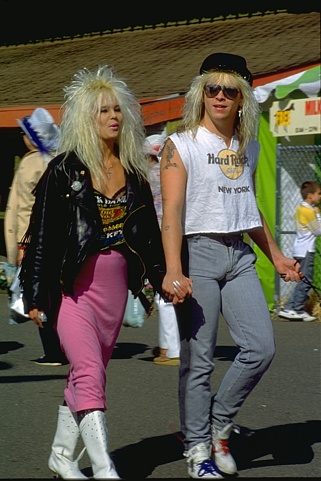}
    \includegraphics[width=0.24\linewidth, frame]{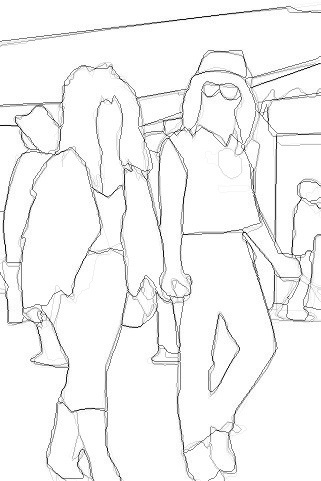}
    \includegraphics[width=0.24\linewidth, frame]{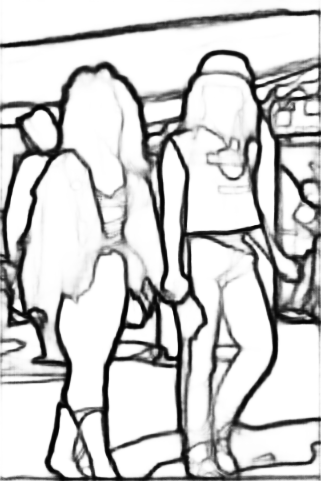}
    \includegraphics[width=0.24\linewidth, frame]{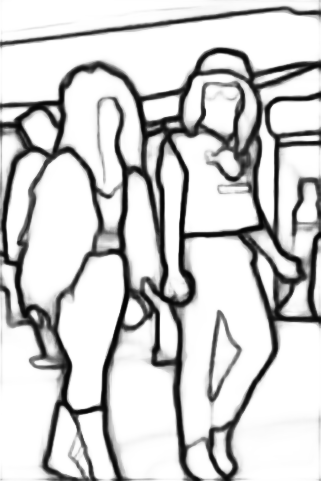}

    \includegraphics[width=0.24\linewidth, frame]{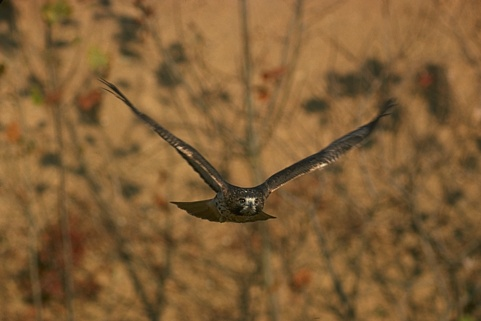}
    \includegraphics[width=0.24\linewidth, frame]{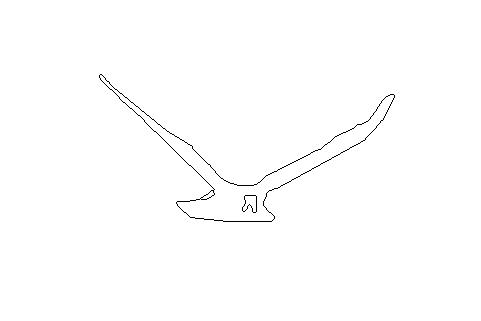}
    \includegraphics[width=0.24\linewidth, frame]{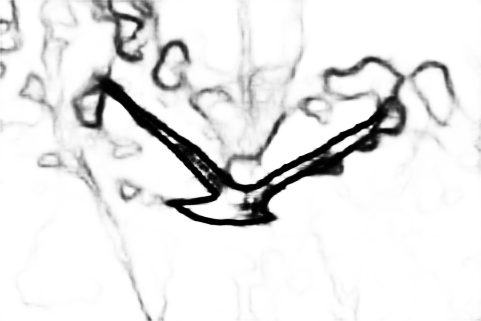}
    \includegraphics[width=0.24\linewidth, frame]{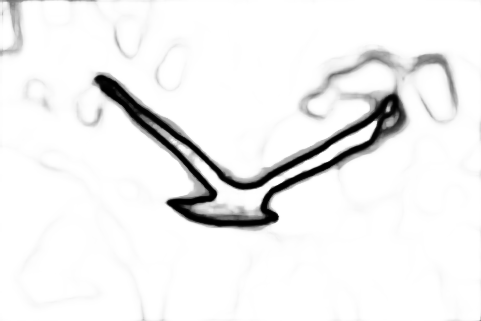}

    \quad\quad Input \quad\quad\quad\;\quad GT \quad\quad\quad Prediction 1 \quad\quad Prediction 2
    \vspace{3pt}
    \caption{Predictions by two different parameter samplings.}
    \label{vis}
\end{figure}

\begin{figure*}[t]
\centering
    \subfloat[Methods w/ pre-training on ImageNet]{\includegraphics[width=0.29\linewidth]{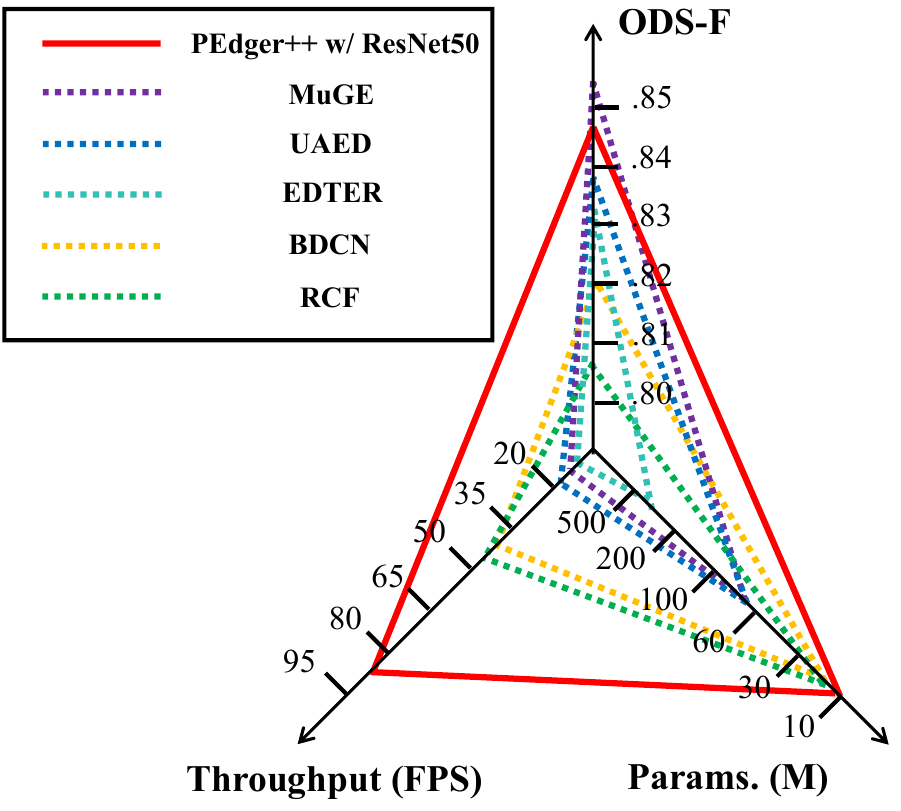}} \quad \quad\quad \subfloat[Methods w/o pre-training]{\includegraphics[width=0.28\linewidth]{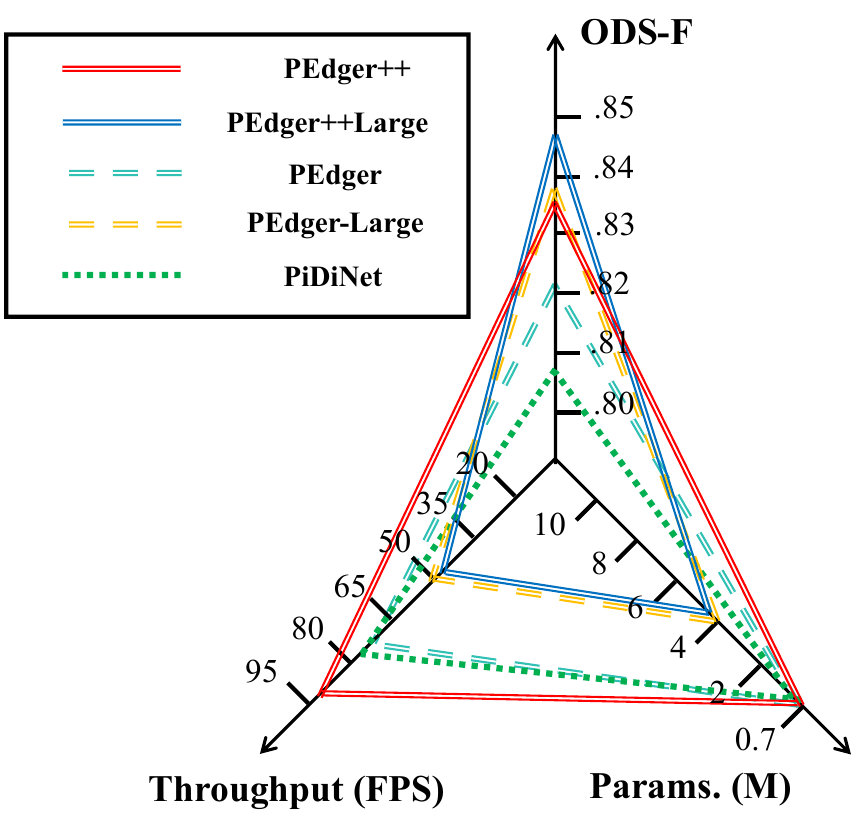}} \quad \quad\quad \subfloat[Our different configurations]{\includegraphics[width=0.31\linewidth]{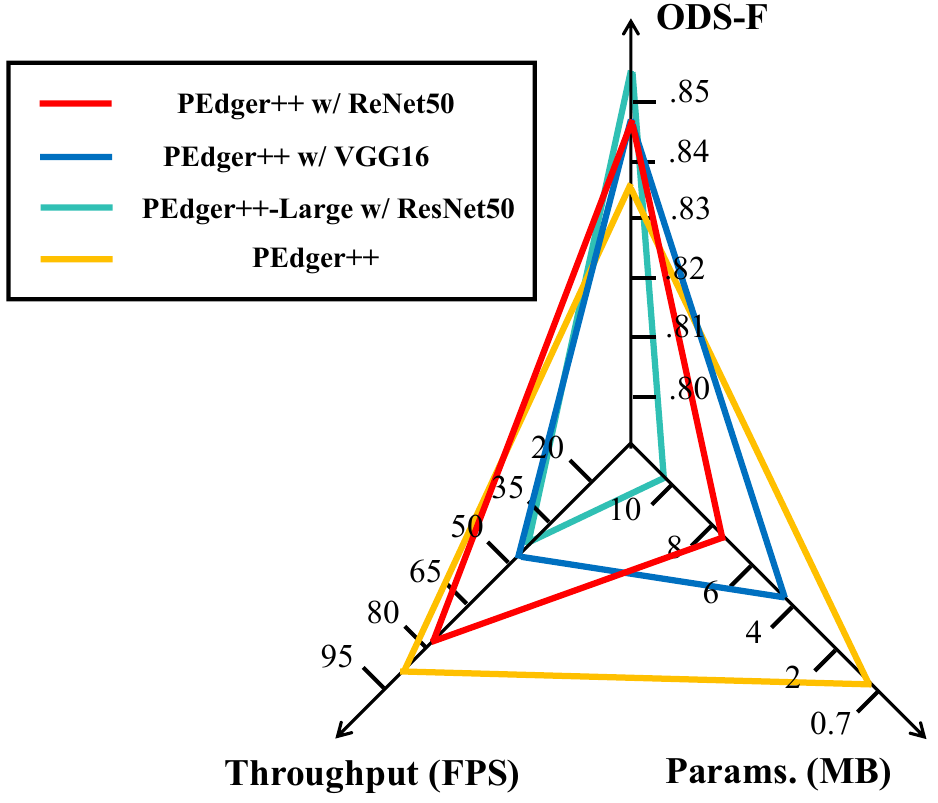}}
    \vspace{5pt}
    \caption{Quantitative comparisons on the BSDS dataset. A larger area of the formed triangle indicates better performance, encompassing three aspects including accuracy (ODS-F), running speed (FPS), and model size (M).}
    \label{tradeoff}
\end{figure*}

Towards advancing the development of an accurate, fast, and lightweight edge detector in the deep learning era, we have to confront a key challenge: \textit{How to explore discriminative and robust information without introducing redundant parameters and expensive computations}? As previously mentioned, shallow or hand-crafted features, such as image gradients, color distributions, and spatial positions, often struggle with discriminability. For instance, pixels with high gradient magnitudes may originate from high-frequency details that are not necessarily edges. Therefore, capturing high-level semantics is beneficial for distinguishing perceptually meaningful edges from misleading details.
But, achieving this goal often necessitates a large amount of network parameters and complex architectures to learn the intricate patterns involved. Consequently, attaining high detection accuracy on devices with limited resources poses a formidable challenge.

To address this challenge, we investigate the integration of useful knowledge--albeit potentially weak--from various sources as a powerful agent. This approach allows us to explore discriminative information through the lens of ensemble learning \cite{Heterogeneous,SELC,ensembles,Generalization,RandomForests}. While knowledge learned from a single source may be biased or unreliable, particularly when the batch size is small, it can still contribute to an effective ensemble. To validate the primary claim of this work, cross-information from multiple sources, including heterogeneous network architectures,  different training moments, and multiple parameter samplings, is utilized in this work for integration/ensemble.

\subsection{Contribution} 

\emph{``Many hands make light work."--John Heywood} 
\\

Inspired by ensemble learning, this work presents a novel framework for edge detection, which integrates knowledge acquired from diverse sources, specifically heterogeneous network architectures, different training moments, and various parameter samplings. First, we incorporate two networks with recurrent (slower) and non-recurrent (faster) structures during training. The knowledge learned from these two architectures is fused using a simple confidence-aware weighting strategy. Second, to assemble information across different training moments, we propose a momentum-based approach to fuse the network parameters corresponding to different epochs. 

Finally, we treat network parameters as random variables following a specific probabilistic distribution, \textit{i.e.},  $p(\mathbf{\Theta}|\mathcal{D})$, to deal with epistemic uncertainty \cite{NEURIPS2020_322f6246,NEURIPS2022_18210aa6,NEURIPS2019_b53477c2,NEURIPS2019_118921ef,NEURIPS2023_da7ce04b,bayesian2,bayesian3}. According to Bayes' theory, the ensemble of stochastic parameter samplings can approximate the Bayesian posterior $p(\mathbf{\Theta}|\mathcal{D})$ over these probabilistic parameters, where $\mathbf{\Theta}$ and $\mathcal{D}$ stand for network parameters and training data, respectively. As illustrated in Fig. \ref{vis}, when sampling different parameters, the model focuses on different aspects/views of the same input, leading to different edge predictions. By combining the comprehensive knowledge from multiple parameter samplings, the robustness of edge detection compared to the biased or uncertain information derived from a single parameter sampling is enhanced, particularly given the ambiguity and subjectivity associated with edge definitions. Figure \ref{vis} shows how different parameter samples from the Bayesian posterior focus on distinct aspects of the same input image,  and output varied edge predictions. Combining these diverse predictions corresponding to differently sampled parameters, yields a more robust and accurate edge detection result. It is noteworthy that, during training, the two networks are updated collaboratively. In other words, they are trained simultaneously without waiting for one to finish before starting another. In the testing phase, only a single forward inference on the faster network is required.

\begin{figure*}[t]
\centering
    \includegraphics[width=0.85\linewidth]{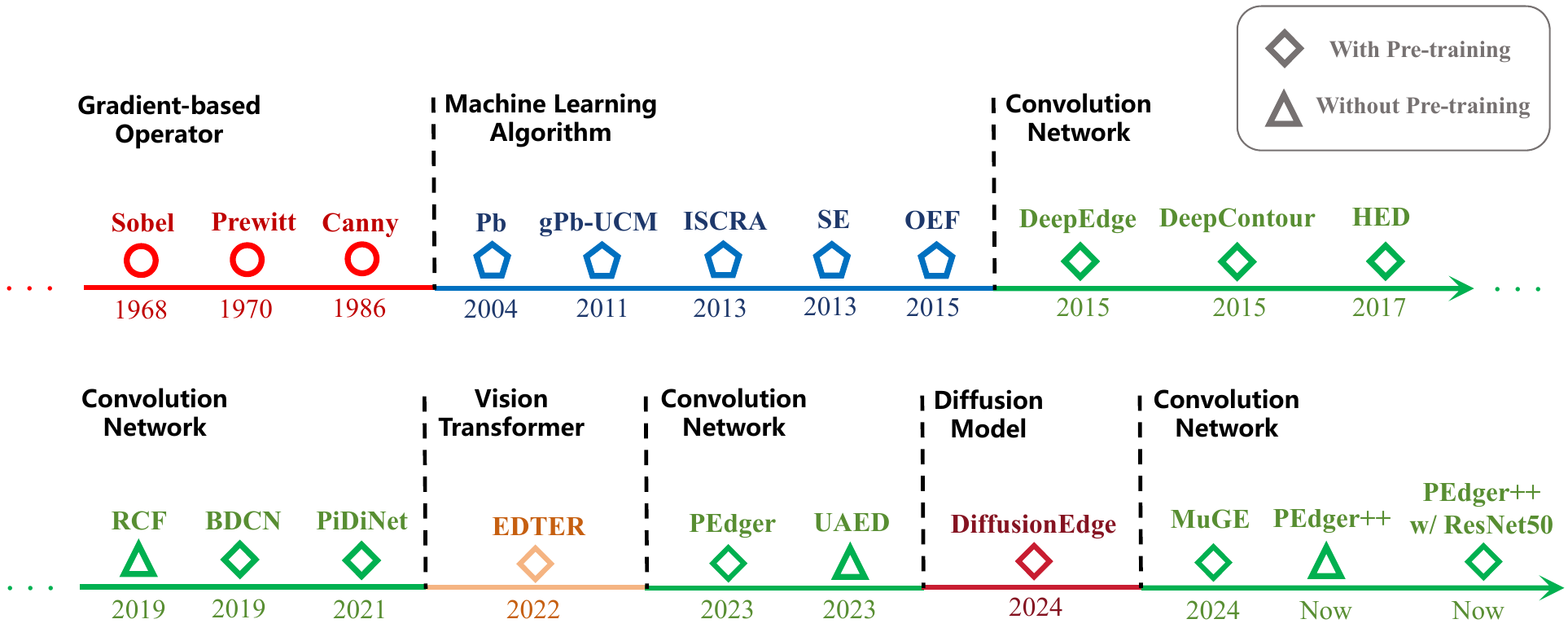}
    
    \caption{Timeline of key developments in edge detection.}
    \label{timeline}
\end{figure*}


A preliminary version of this manuscript appears in \cite{PEdger}. Compared to our previous PEdger \cite{PEdger}, which  focuses solely on integrating across network structures and training moments, this journal version designs a novel algorithm to approximate Bayesian posteriors through parameter ensembles, dealing with the epistemic uncertainty. Unlike traditional Bayesian learning methods, we determine the influence/weight of each parameter sampling based on its generalization ability on unseen data. Moreover, we avoid using inefficient batch normalization in PEdger, opting instead for adaptive gradient clipping \cite{GradientClipping} to accelerate the process. The above improvements result in significant performance gains in PEdger++ (\emph{e.g.}, on the BSDS dataset, ODS-F 0.835, Throughput 92FPS, Parameter 716K), surpassing SOTA competitors like PEdger (ODS-F 0.821, Throughput 71FPS, Parameter 734K) and PiDiNet (ODS-F 0.807, Throughput 76FPS, Parameter 710K).  

Besides, we provide more extensive ablation studies in this journal paper to evaluate the efficacy of our design, and offer a broader range of model versions with varying computational costs. For instance, by increasing the amount of model parameters, PEdger++Large further boosts the accuracy to ODS-F 0.848 with Throughput 48.5FPS and Parameter 4.3M. Additionally, when incorporating pre-training as in most previous methods, our PEdger++ w/ ResNet50 and PEdger++Large w/ ResNet50 achieve ODS-F 0.846 (Throughput 79FPS, Parameter 7.8M) and ODS-F 0.857 (Throughput 37.6FPS, Parameter 12.7M), respectively. For more details, please refer to Fig.~\ref{tradeoff}. These results demonstrate that our method can be flexibly adjusted according to different user and application demands.

\section{Related Work}
\label{sec:related}
\subsection{Edge Detection}
    Edge detection is a fundamental problem of image processing and computer vision, with extensive research efforts dedicated to this field, as depicted in Fig.~\ref{timeline}. Despite significant advancements, balancing accuracy and efficiency remains a persistent challenge. Early edge detection techniques originated from gradient-based operators such as Sobel \cite{sobel}, Prewitt \cite{Prewitt}, and Canny \cite{Canny}, which detect edges by computing intensity gradients and emphasizing regions with abrupt changes. To refine edge detection and suppress noise or irrelevant details, the Laplacian of Gaussian (LoG) operator \cite{smith1988edge} applies Gaussian smoothing before calculating second-order gradients. While being computationally efficient, these traditional gradient-based methods often struggle to deliver the accuracy required by real-world applications, particularly in handling complex or noisy images. 

Alternatively, many researchers have explored machine learning approaches for edge detection as a distinct category. For example, pb \cite{pb} and gpb-UCM \cite{gpu-ucm} integrate oriented gradient features from three brightness channels, along with a texture channel, into a logistic regression classifier. ISCRA \cite{ISCRA} segments input images into multiple regions and then gradually merges them hierarchically. Structured Edges (SE) \cite{se} uses a random decision forest to detect edges, leveraging the intrinsic structure of local image patches. Hallman \textit{et al.} \cite{oef} further improved edge detection by integrating efficient clustering of training samples and calibrating individual tree output probabilities, yielding better results than SE \cite{se}. However, despite their increased accuracy, these machine learning-based methods are limited by their dependence on hand-crafted features and complex learning pipelines, which restrict their potential for further performance improvements.

Amidst the rise of deep learning techniques, high-level semantic features can be learned by deep neural networks, significantly enhancing edge detection performance. Generally, based on the classification requirements, deep edge detectors are categorized into two types: binary pixel-level classification and multi-pixel-level classification (\textit{i.e.}, semantic edge detection \cite{SEAL,STEAL,DFF,Casenet,PNTEdge,DBLP:journals/pr/PanJZLT24}). This work mainly concentrates on deep learning-based binary edge detection. 
As a representative in this category, Ganin \textit{et al.} \cite{N4Fields} employed nearest neighbor search in conjunction with neural network predictions to tackle the under-fitting problem, which has been proven to be beneficial for edge detection and thin object segmentation. DeepEdge \cite{deepedge} leverages object-related features to guide contour detection, arguing that object recognition and contour identification are highly correlated tasks. Almost the same time, DeepContour \cite{deepcontour} divides the contour class into multiple sub-classes, each trained with distinct parameters to improve edge detection precision. HFL \cite{HFL}, which uses the VGG16 \cite{VGG16} model pre-trained on ImageNet \cite{ImageNet} as its backbone, demonstrates how detected edges can enhance various high-level vision tasks. COB \cite{Man+16a,COB} incorporates gradient orientations and produces hierarchical oriented contours, improving the quality of edge detection. To ensure a well-guided optimization during training, HED \cite{HED} and RCF \cite{RCF} yield multi-scale side-outputs, where each output is supervised by the corresponding ground truth. While LPCB \cite{LPCB} and CED \cite{CED} adopt bottom-to-top and top-to-bottom information propagation strategies to facilitate interaction between outputs of different scales, which help generate crisp edges. However, a notable limitation of these approaches is that the side-outputs are trained under identical supervision, ignoring the intrinsic diversity between different scales.

To address the aforementioned issues, various studies \cite{relaxed,BDCN_cvpr,EDTER,treasure,MuGE} have imposed diverse supervisions on side-outputs. Besides, several works have explored the use of Vision Transformers (ViT) and diffusion models for edge detection, yielding promising results, as demonstrated by EDTER \cite{EDTER} and DiffusionEdge \cite{Diffusionedge}.  Recently, UAED \cite{treasure} employs a dual-branch structure to handle annotation uncertainty, recognizing that divergence among annotators can affect edge detection. One branch focuses on detecting edges, while the other quantifies the uncertainty caused by annotation variations. MuGE \cite{MuGE} generates multiple edge maps with varying granularities by decomposing feature maps into low-frequency and high-frequency components. While these deep learning-based methods substantially improve edge detection accuracy, they often come with the trade-off of redundant network parameters and complex operations, resulting in increased computational costs. In response, PiDiNet \cite{PiDiNet_TPAMI} introduces pixel difference convolution, an efficient technique for extracting edge-relevant features. Compared to previous methods, PiDiNet reduces model size and accelerates processing speed, but there remains large room for improvement in accuracy.

This work seeks to advance the field of edge detection by developing an innovative edge detector that achieves high accuracy, rapid processing speed, and small model size. Unlike previous deep learning-based edge detectors, our method delivers superior accuracy across varying levels of computational complexity. This advancement is underpinned by the ability to avoid redundant network parameters and inefficient computational processes, ensuring both efficiency and precision.  

\subsection{Uncertainty Modeling}
Uncertainty modeling is a critical problem in machine learning, essential for evaluating the reliability of models at specific data points \cite{DBLP:conf/nips/HuangLZ23,DBLP:conf/nips/SchweighoferAIK23,DBLP:conf/nips/WarburgMBH23}. Uncertainty can be broadly categorized into two types: aleatoric and epistemic uncertainty \cite{uncertainties}.  Aleatoric uncertainty pertains to the inherent noise in the data itself. Data of poor quality or corrupted by stochastic noise results in higher aleatoric uncertainty, often modeled using techniques such as Dirichlet priors  \cite{DBLP:conf/nips/MalininG18,DBLP:conf/nips/MalininG19}, or post-training approaches \cite{DBLP:conf/nips/SensoyKK18,DBLP:conf/icml/GuoPSW17,DBLP:conf/nips/KullPKFSF19}. On the other hand, epistemic uncertainty refers to the uncertainty in the model's parameters, indicating the model’s limited knowledge of the input distribution. High epistemic uncertainty reflects inadequate understanding of the data. Bayesian deep learning \cite{NEURIPS2020_322f6246,NEURIPS2022_18210aa6,NEURIPS2019_b53477c2,NEURIPS2019_118921ef,10302334,bayesian2,bayesian3,9412521,backpropagation,DBLP:phd/ca/Neal95} has emerged as an effective approach to model epistemic uncertainty, which, as a well-established technique in computer vision and multimedia fields, has a rich history of applications \cite{prior,action,fast,Libre,volumetric,registration,audiovisual,9924527,bayesian3,LIANG2023109810}, offering robust solutions for modeling epistemic uncertainty. In this context, network parameters $\mathbf{\mathbf{\Theta}}$ are treated as random variables, and the posterior distributions of these parameters over training data $\mathcal{D}$ can be derived using Bayes' theorem: ${p}(\mathbf{\mathbf{\Theta}}|\mathcal{D}) \defeq {p}(\mathcal{D} | \mathbf{\Theta}) {p}(\mathbf{\Theta}) / {p}(\mathcal{D})$. 

Unfortunately, it is impractical to directly compute the analytical Bayesian posteriors ${p}(\mathbf{\Theta}|\mathcal{D})$. Hence, methods for approximating Bayesian inference become essential for estimating the posterior distribution of parameters \cite{NEURIPS2023_5d97b7e6,NEURIPS2023_7d25b1db,NEURIPS2022_12143893,Qu_2021_ICCV,9121755,Krishnan_Subedar_Tickoo_2020,9946419}. One popular approach is variational inference, which approximates Bayesian inference by finding simpler distributions that minimize divergence from the true posterior \cite{variational1,variational2}. Gal \textit{et al.} \cite{DropoutBayesian} further applied Dropout \cite{dropout} to neural network parameters to approximate variational inference in Gaussian processes. More recently, DropConnect \cite{DropConnect} quantifies the model parameter uncertainty via directly imposing Bernoulli distribution on network parameters, without the need of additional trainable weights. Teye \textit{et al.} \cite{batch} explored the stochasticity arising from randomly chosen mini-batch instances, interpreting it as a proxy for Bayesian inference. Lakshminarayanan \textit{et al.} \cite{scalable} investigated the use of ensembles and adversarial examples in Bayesian learning, offering a scalable alternative. In parallel, several works \cite{stochastic,DBLP:journals/pami/HansenS90,DBLP:conf/nips/Lakshminarayanan17} exploit deep stochastic ensembles to approximate the Bayesian posterior. 

Despite demonstrating encouraging performance, previous literature has rarely addressed how to reasonably assess the influence of each parameter sampling when combining or fusing multiple parameter samplings. This limitation constrains the potential for further accuracy improvements. To overcome this challenge, this paper proposes an effective ensemble scheme for estimating the Bayesian posterior based on their generalization ability, better measuring the epistemic uncertainty in edge detection.
\begin{figure}[t]
    \centering
    \subfloat[Recurrent Architecture]{
    \includegraphics[width=\linewidth]{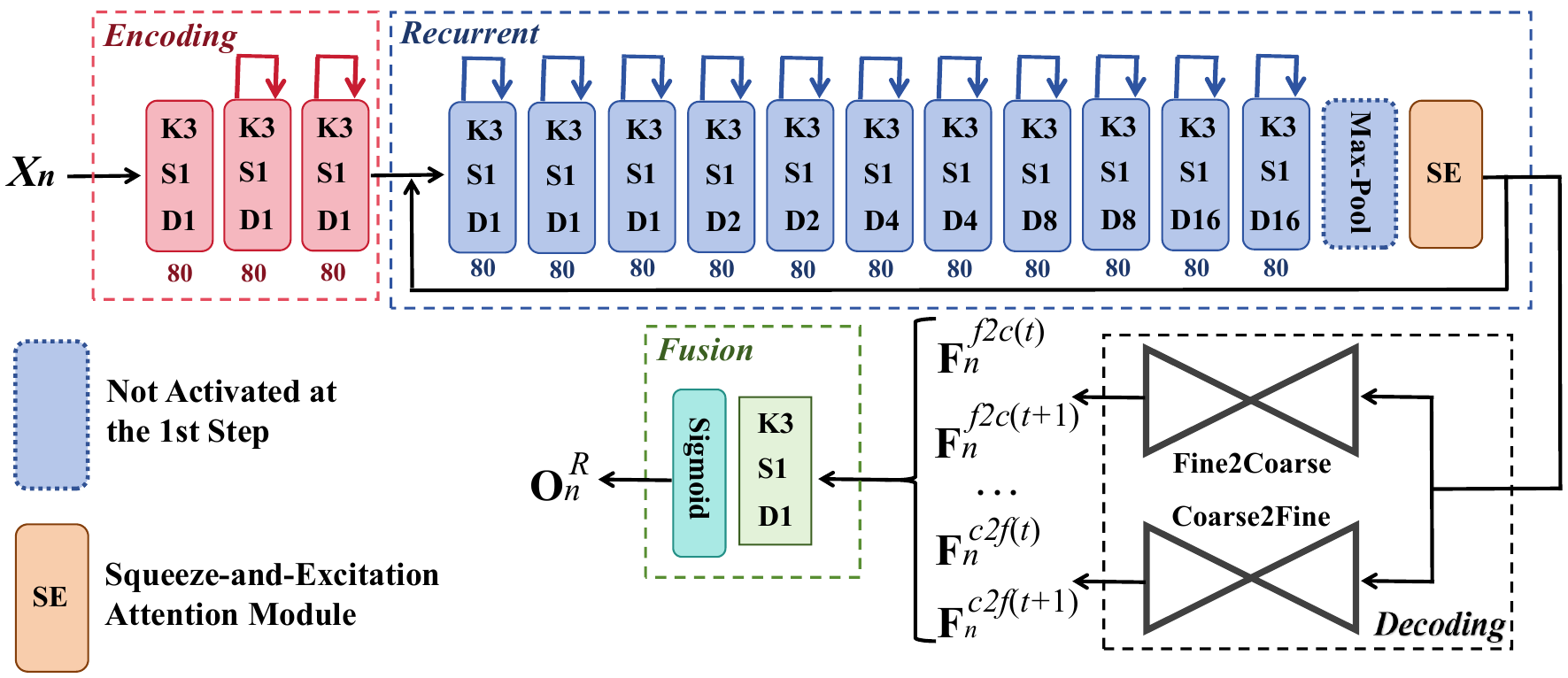}
    }
    
    \subfloat[Non-recurrent Architecture]{
    \includegraphics[width=\linewidth]{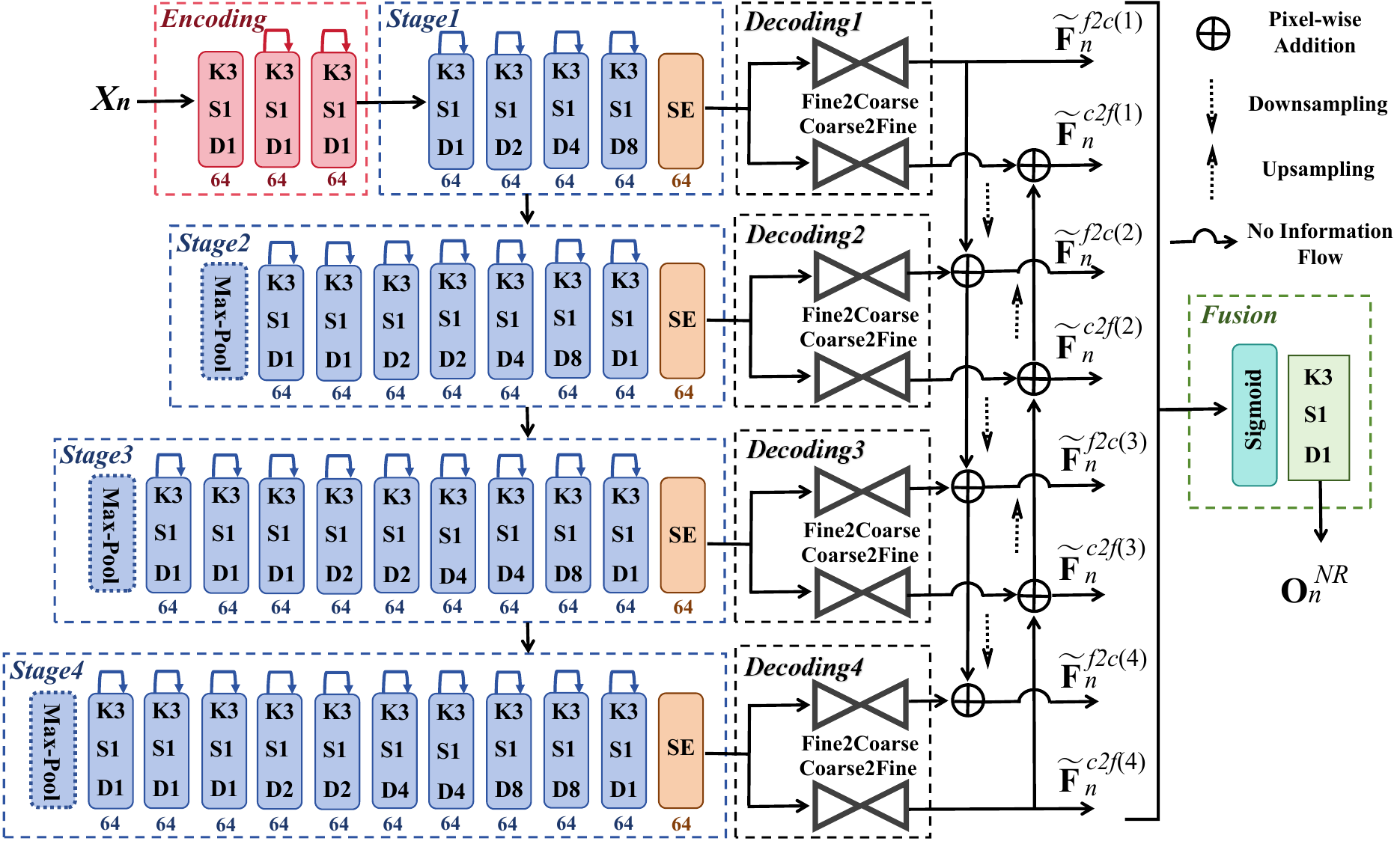}
    }
    
    \caption{Recurrent and non-recurrent architectures. $K$, $S$, and $D$ are the kernel size, stride, and dilation rate, respectively. The number of output channels is shown below each block.}
    \label{arch}
\end{figure}

\section{Methodology}
This work aims to devise a practical edge detector from an ensemble perspective. Inspired by principles from Error-Ambiguity decomposition (analogous to Bias-Variance decomposition and others) \cite{ensembles,Generalization,RandomForests}, a combination of various agents can enhance overall accuracy. The generalization error $\mathbf{E}$ of an ensemble depends on two primary factors, represented as follows:
\begin{equation}
    \mathbf{E}= \bar{\mathbf{E}} - \bar{\mathbf{A}}, 
    \label{eq:ensemble}
\end{equation}
where $\mathbf{\bar{E}}$ designates the average error of the individual predictors, and $\mathbf{\bar{A}}$ is the diversity/difference between the ensemble and its components. Hence, we propose assembling cross-information from various sources, including heterogeneous network architectures, different training moments, and multiple parameter samplings in this work, so as to achieve both high accuracy and diversity.


\subsection{Heterogeneous Network Architectures}
\label{NetworkArchitecture}
Our proposed framework encompasses very simple recurrent and non-recurrent networks\footnote{The primary contribution of this work lies in the cross-information assembly rather than the network architecture itself. Thus, to minimize the impact of complex network designs and to validate our main objective, we intentionally utilize straightforward architectures.}, each initialized with distinct random parameters, to capture diverse features and enhance robustness. The two networks are collaboratively trained in a single-stage process. This ensemble strategy is grounded in the effectiveness of integrating heterogeneous models to improve resilience, as evidenced by prior studies on model ensembles \cite{ensembles,Generalization,RandomForests}. During the testing phase, only the non-recurrent network is executed. Notably, to accelerate inference speed, we avoid the inefficient batch normalization and instead apply adaptive gradient clipping \cite{GradientClipping}, which dynamically clips back-propagation gradients based on the ratio of gradient norms to parameter norms. This clipping strategy enables the training of deep neural networks without the need for normalization layers. As a consequence, our PEdger++ framework possesses faster inference speeds than the previous PEdger model \cite{PEdger}.

\textbf{Recurrent Architecture.} The blueprint of our recurrent network $\mathcal{G}^R(\cdot)$ is illustrated in Fig. \ref{arch} (a). For each data sample $\mathbf{X}_n$, an initial encoding module--consisting of a convolution layer and two residual blocks--processes the input, producing a feature set. These features are then fed into a recurrent module with parameters shared across all recurrent steps. In detail, at step $t$ ($t > 1$), the recurrent module receives as input the features generated at step $(t-1)$. For the first recurrent step ($t = 1$), it directly takes the features from the encoding module. Each step's output then serves as the input for the following step, continuing iteratively until the maximum step count $T$ is reached.

With the exception of the first step, the recurrent module is connected with a max-pooling operation at each subsequent recurrent step. As demonstrated by \cite{HED,RCF,BDCN_cvpr}, using max-pooling progressively enlarges the receptive fields, promoting multi-scale side-output edges. As steps advance, the information captured becomes increasingly coarse, which strengthens responses for larger objects due to the expanded receptive field size. Additionally, inspired by works like \cite{LPCB,CED,BDCN_cvpr} that aggregate information bi-directionally (fine-to-coarse and coarse-to-fine), we introduce two parallel branches in the decoding module (depicted in Fig. \ref{arch}). 
These branches aggregate features across steps in both directions: the fine-to-coarse branch accumulates features from earlier (finer) to later (coarser) steps, and the coarse-to-fine branch from later steps to earlier ones.
Concretely, at each step $t$, the side-output features $\mathbf{F}_n^{f2c(t)}$ are computed by adding outputs from the fine-to-coarse branch to the down-sampled side-output features from the previous step, $\mathbf{F}_n^{f2c(t-1)}$. Similarly, $\mathbf{F}_n^{c2f(t)}$ is obtained by adding outputs from the coarse-to-fine branch with up-sampled features $\mathbf{F}_n^{c2f(t+1)}$ from the next step. This recurrent operation can be formally expressed as:
\begin{equation}
\mathbf{F}_{n}^{f2c(t)} = 
\begin{cases} 
\mathbf{Z}_n^{f2c(t)}, & t=1, \\
\mathbf{Z}_n^{f2c(t)} + \text{Down}\left(\mathbf{F}_{n}^{f2c(t-1)}\right), & t > 1 ,
\end{cases}
\end{equation}
\begin{equation}
\mathbf{F}_{n}^{c2f(t)} = 
\begin{cases} 
\mathbf{Z}_n^{c2f(t)}, & t=T, \\
\mathbf{Z}_n^{c2f(t)} + \text{Up}\left(\mathbf{F}_{n}^{c2f(t+1)}\right), & t < T, 
\end{cases}
\end{equation}
where $\mathbf{Z}_n^{f2c(t)}$ and $\mathbf{Z}_n^{c2f(t)}$ are the features generated by the fine-to-coarse branch and coarse-to-fine branch at the $t$-th recurrent step for the $n$-th sample, in the decoding module (black part in Fig. \ref{arch} (a)), respectively. $\text{Down}(\cdot)$ and  $\text{Up}(\cdot)$ denote the down-sampling and up-sampling operation, respectively.
Finally, all side-output features are fused through a $1\times 1$ convolution, and processed with a Sigmoid activation to produce the edge detection output $\mathbf{O}_n^R$. Thanks to shared parameters, the recurrent network is remarkably compact and converges quickly. The fusion process can therefore be described as:
\begin{equation}
\mathbf{O}_{n}^{R} = \sigma\left(\mathcal{C}_{1\times 1}\left(\left[\mathbf{F}_{n}^{f2c(1)}, \mathbf{F}_{n}^{c2f(1)}, ..., \mathbf{F}_{n}^{f2c(T)}, \mathbf{F}_{n}^{c2f(T)}\right]\right)\right),
\end{equation}
where $\sigma(\cdot)$ means the Sigmoid activation function. $\mathcal{C}_{1\times 1}$ is the $1\times 1$ convolution layer. $\left[\cdot, ..., \cdot\right]$ represents the concatenation operation along the channel dimension. It is worth mentioning that, the recurrent structure employs a greater number of channels than the non-recurrent counterpart, reflecting efficient parameter usage to expand channel.

\begin{algorithm}[t]
    \caption{PEdger++ via Collaborative Learning}
    \label{alg:algorithm}
    \raggedright
    \textbf{Input}: Training set $\mathcal{D}_t$ and validation set $\mathcal{D}_v$, split from $\mathcal{D}$; the maximum number of parameter samplings $S$; total epochs $J$; back-propagation networks  $\mathcal{G}^R$ and $\mathcal{G}^{NR}$, and momentum networks $\mathcal{G}_m^R$ and $\mathcal{G}_m^{NR}$ \\ 
    
    \textbf{Output}: Final network $\mathcal{G}_m^{NR}(\cdot; \mathbf{\Theta}^{NR})$\\
    
    \textbf{Initialize}: $j\defeq0$; $\eta_{J}\defeq0.8$; $\mu\defeq 0.5$, $\lambda$ (varied for different datasets), and initialized $\mathcal{G}^R$ and $\mathcal{G}^{NR}$
    \begin{algorithmic}[1] 
        \WHILE{$j < J$}
        \STATE $n_t\defeq 1$;
        \STATE Determine $\eta_{j}$ via Eq. \eqref{weightincrease};
         \WHILE{$n_t < N_t$}
        \STATE Sample a data pair ($\mathbf{X}_{n_t}$, $\mathbf{\mathbf{Y}}_{n_t}$) from $\mathcal{D}_t$;
        \IF {$j > 0$}
        \STATE Compute $\mathbf{M}_{n_t}$ via Eqs. \eqref{eq:Mnt} and \eqref{uncertainty};
        \STATE Refine $\tilde{\mathbf{Y}}_{n_t}$ via Eq. \eqref{soft};
        \ELSE
        \STATE  $\tilde{\mathbf{Y}}_{n_t} \defeq \mathbf{\mathbf{Y}}_{n_t}$;
        \ENDIF
        \STATE Update $\mathcal{G}^R$ and $\mathcal{G}^{NR}$ through back-propagation with respect to Eqs. \eqref{loss1} and \eqref{loss2}, respectively; 
        \STATE $n_t \defeq n_t + 1$;
        \ENDWHILE
        \IF {$j > 0$}
        \STATE Update $\mathcal{G}_m^R$ and $\mathcal{G}_m^{NR}$ via Eq. \eqref{momentum};
        \ELSE
        \STATE Copy the parameters from $\mathcal{G}^R$, $\mathcal{G}^{NR}$ to $\mathcal{G}_m^R$, $\mathcal{G}_m^{NR}$;
        \ENDIF
        \STATE Sample parameters $\mathbf{\Theta}_s^R \text{ and } \mathbf{\Theta}_s^{NR}, s \in \{1, ..., S\}$;
        \STATE Update $\mathbf{W}_s^{R} \text{ and } \mathbf{W}_s^{NR}, s \in \{1, ..., S\}$, via Eqs. \eqref{optimal_weight_R} and \eqref{optimal_weight_NR};
        \STATE $j \defeq j + 1$;
        \ENDWHILE
    \end{algorithmic}
\end{algorithm}

\begin{algorithm}[t]
    \caption{Efficient Version of Collaborative Learning}
    \label{alg:algorithm2}
    \raggedright
    \textbf{Input}: Training set $\mathcal{D}_t$ and validation set $\mathcal{D}_v$, split from $\mathcal{D}$; the maximum number of parameter samplings $S$; total epochs $J$; the non-recurrent back-propagation network $\mathcal{G}^{NR}$, and its corresponding momentum networks $\mathcal{G}_m^{NR}$ \\ 
    
    \textbf{Output}: Final network $\mathcal{G}_m^{NR}(\cdot; \mathbf{\Theta}^{NR})$\\
    
    \textbf{Initialize}: $j\defeq0$; $\eta_{J}\defeq0.8$; $\mu\defeq 0.5$, $\lambda$ (varied for different datasets), and initialized $\mathcal{G}^{NR}$
    \begin{algorithmic}[1] 
        \WHILE{$j < J$}
        \STATE $n_t\defeq 1$;
        \STATE Determine $\eta_{j}$ via Eq. \eqref{weightincrease};
         \WHILE{$n_t < N_t$}
        \STATE Sample a data pair ($\mathbf{X}_{n_t}$, $\mathbf{\mathbf{Y}}_{n_t}$) from $\mathcal{D}_t$;
        \IF {$j > 0$}
        \STATE Compute $\mathbf{M}_{n_t}$ via Eqs. \eqref{eq:Mnt} and \eqref{uncertainty};
        \STATE Refine $\tilde{\mathbf{Y}}_{n_t}$ via Eq. \eqref{soft};
        \ELSE
        \STATE  $\tilde{\mathbf{Y}}_{n_t} \defeq \mathbf{\mathbf{Y}}_{n_t}$;
        \ENDIF
        \STATE Update $\mathcal{G}^{NR}$ through back-propagation with respect to Eq. \eqref{loss2}, respectively; 
        \STATE $n_t \defeq n_t + 1$;
        \ENDWHILE
        \IF {$j > 0$}
        \STATE Update $\mathcal{G}_m^{NR}$ via Eq. \eqref{momentum};
        \ELSE
        \STATE Copy the parameters from $\mathcal{G}^{NR}$ to $\mathcal{G}_m^{NR}$;
        \ENDIF
        \STATE Sample parameters $\mathbf{\Theta}_s^{NR}, s \in \{1, ..., S\}$ via stochastic weight pruning;
        \STATE Update $\mathbf{W}_s^{NR}, s \in \{1, ..., S\}$, via Eq. \eqref{optimal_weight_NR};
        \STATE $j \defeq j + 1$;
        \ENDWHILE
    \end{algorithmic}
\end{algorithm}

\textbf{Non-recurrent Architecture.} Recurrent models often exhibit slower inference speeds due to the repeated processing of a recurrent module (blue part in Fig.\ref{arch} (a)). To construct a faster network $\mathcal{G}^{NR}(\cdot)$, we allocate varying parameter amounts across different scales\footnote{Multi-scale strategies can also be viewed as an ensemble of features from different scales.} (blue part in Fig.\ref{arch} (b)). In particular, parameter allocation follows a progressive pattern: the first stage has the fewest parameters, with the parameter amount increasing as the feature size decreases in subsequent stages. This allocation is purposefully structured to optimize feature extraction at each stage. Earlier stages prioritize capturing low-level features, which are characterized by larger spatial resolutions, while later stages focus on generating high-level semantic representations at lower resolutions.
 The rationale behind is that high-level semantic structures, which emerge in later stages, require a greater number of parameters for effective learning compared to the lower-level details processed in earlier stages. Moreover, as feature maps reduce in resolution across stages, processing later stages requires less computational overhead, allowing for efficient use of increased parameters. Thus, the parameter count scales up as the resolution scales down, contributing to faster inference speeds.

As shown in Fig. \ref{arch}, the primary distinction between the recurrent and non-recurrent structures lies in whether the parameters of the modules indicated in the blue and black (the decoding module) are shared across multiple scales. Notice that the non-recurrent architecture also incorporates bidirectional aggregation to produce coarse-to-fine features $\widetilde{\mathbf{F}}_n^{c2f}$ and fine-to-coarse ones $\widetilde{\mathbf{F}}_n^{f2c}$, with the same operational mechanism as in the recurrent structure.

\subsection{Different Training Moments}
Generally, deep neural networks learn different types of features as training unfolds. At early training moments, networks tend to capture broad, generalizable patterns; As training advances, they shift toward more complex and specific knowledge to handle harder samples.
However, this shift often introduces an increased risk of overfitting to noisy or detrimental information. 
Therefore, integrating knowledge across the entire training timeline can mitigate the adverse effects of biases or unreliable features learned at a single training moment. In this work, we propose to fuse knowledge learned at different training moments throughout the whole training process, via updating the parameters of a momentum network as follows:
\begin{equation}
\begin{aligned}
{\mathbf{\Theta}}_m^{(j)} \defeq \mu\cdot{\mathbf{\Theta}}_{bp}^{(j)} + (1-\mu)\cdot{\mathbf{\Theta}}_m^{(j-1)},
\end{aligned}
\label{momentum}
\end{equation}
where ${\mathbf{\Theta}}_m^{(j)}$ and ${\mathbf{\Theta}}_{bp}^{(j)}$ are the parameters of a momentum network, and its corresponding network updated by back-propagation, after epoch $j$, respectively. In addition, $\mu\in[0,1]$ is a balancing hyper-parameter, which is empirically set to $0.5$ in this work. During the training, both the recurrent $\mathcal{G}^R_m(\cdot)$ and non-recurrent $\mathcal{G}^{NR}_m(\cdot)$ momentum networks are maintained. 

\subsection{Multiple Parameter Samplings}
\label{Bayesian}
As previously mentioned, epistemic uncertainty can be addressed through Bayesian neural networks, where model parameters are represented as Bayesian posteriors $p(\mathbf{\Theta}|\mathcal{D})$.  Since $p(\mathbf{\Theta}|\mathcal{D})$ is unknown and direct computation is impractical, we approximate it by using a parameter ensemble that captures cross information from multiple sampled network parameters. Notably, prior methods \cite{bayesian1,subnetwork,bit,uncertainties,stochastic} for approximating Bayesian posteriors via parameter ensembles, typically determine the influence of each parameter sample based solely on training data, without considering its generalization capacity on unseen data. We argue that the weight assigned to each sample in the ensemble should reflect its generalization ability, ensuring that samples with limited generalization are down-weighted while those with stronger generalization are emphasized.

To achieve a well-weighted ensemble of differently sampled parameters, we propose selecting the weights based on generalization ability assessed on a validation set. We begin by stochastically partitioning the original training set $\mathcal{D}$ into a new training set,  $\mathcal{D}_t \defeq \{(\mathbf{X}_{n_t}, \mathbf{\mathbf{Y}}_{n_t}), n_t \in \{1, 2, ..., N_t\} \}$, and a validating set $\mathcal{D}_v \defeq \{(\mathbf{X}_{n_v}, \mathbf{Y}_{n_v}), n_v \in \{1, 2, ..., N_v \} \}$, where $N_t + N_v = N$. To find optimal weights for each sampled parameter, we solve the following problems:
\begin{equation}
\begin{aligned}
\argmin_{\mathbf{W}^R_s, s \in \{1, ..., S\} } & \quad \frac{1}{N_v}\sum_{n_v\defeq 1}^{N_v} \mathcal{L}(\mathbf{M}_{n_v}^R, \mathbf{Y}_{n_v}),
\\
\argmin_{\mathbf{W}^{NR}_s, s \in \{1, ..., S\} } & \quad \frac{1}{N_v}\sum_{n_v\defeq 1}^{N_v} \mathcal{L}( \mathbf{M}_{n_v}^{NR}, \mathbf{Y}_{n_v}),\\
\end{aligned}
\end{equation}
where $S$ is the total number of parameter samplings. $\mathbf{X}_{n_v}$ and $\mathbf{Y}_{n_v}$ denote the $n_v$-th input image and its corresponding ground truth in the validation set $\mathcal{D}_v$, respectively. $\mathbf{M}_{n_v}^R$ and $\mathbf{M}_{n_v}^{NR}$ are generated by:
\begin{equation}
\begin{aligned}
\mathbf{M}_{n_v}^R &\defeq \sum_{s\defeq 1}^{S} \mathbf{W}_s^{R} \circ  \mathcal{G}_m^R(\mathbf{X}_{n_v};\mathbf{\Theta}^R_s), 
\\
\mathbf{M}_{n_v}^{NR} &\defeq \sum_{s\defeq 1}^{S} \mathbf{W}_s^{NR} \circ \mathcal{G}_m^{NR}(\mathbf{X}_{n_v}; \mathbf{\Theta}^{NR}_s), 
\end{aligned}
\label{eq:Mnv}
\end{equation}
where $\mathbf{\Theta}^R_s$ and $\mathbf{\Theta}^{NR}_s$ indicate the $s$-th sampled network parameters of the recurrent $\mathcal{G}_m^R(\cdot; \mathbf{\Theta}^R_s)$ and non-recurrent $\mathcal{G}_m^{NR}(\cdot; \mathbf{\Theta}^{NR}_s)$ momentum networks, respectively. The symbol $\circ$ means the Hadamard product, while $\mathbf{W}_s^R$ and $\mathbf{W}_s^{NR}$ denote the pixel-wise non-negative ensemble weights of the $s$-th parameter sample of the recurrent and non-recurrent network, respectively. These weights are regularized by $\sum_{s \defeq 1}^{S} \mathbf{W}^R_s = \mathbf{1}$ and $\sum_{s \defeq 1}^{S} \mathbf{W}^{NR}_s = \mathbf{1}$, where $\mathbf{1}$ is an all-one matrix of the same dimensions as $\mathbf{Y}_{n_v}$.  We use the binary cross entropy $\mathcal{L}( \cdot, \cdot)$ as defined in Eq. \eqref{BCE}, to quantify  the distance between detected edges and the ground truth. 

The optimal weight of the $s$-th parameter sampling can be determined by the Lagrange multiplier method as:
\begin{equation}
\begin{aligned}
\mathbf{W}^{R}_s \defeq \mathbf{C} \circ \sum_{n_v\defeq 1}^{N_v} \mathcal{G}_m^{R}(\mathbf{X}_{n_v}; \mathbf{\Theta}^{R}_s) \circ (\frac{\mathbf{Y}_{n_v}}{\mathbf{M}_{n_v}^R} - \frac{\mathbf{1}-\mathbf{Y}_{n_v}}{\mathbf{1}-\mathbf{M}_{n_v}^R}),\\
\end{aligned}
\label{optimal_weight_R}
\end{equation}
\begin{equation}
\begin{aligned}
\mathbf{W}^{NR}_s \defeq \mathbf{C} \circ \sum_{n_v\defeq 1}^{N_v}\mathcal{G}_m^{NR}(\mathbf{X}_{n_v}; \mathbf{\Theta}^{NR}_s)\circ (\frac{\mathbf{Y}_{n_v}}{\mathbf{M}_{n_v}^{NR}} - \frac{\mathbf{1}-\mathbf{Y}_{n_v}}{\mathbf{1}-\mathbf{M}_{n_v}^{NR}}),\\
\end{aligned}
\label{optimal_weight_NR}
\end{equation}
with $\mathbf{C} \defeq \frac{N_v}{2 \sum_{n_v\defeq 1}^{N_v}\mathbf{Y}_{n_v}}$.

\begin{table*}[th!]
\centering
\caption{Performance comparison only on the augmented BSDS data for training. Running speeds are benchmarked on an RTX 3090 GPU with input images sized 320 $\times$ 480 images unless stated otherwise. $\dagger$ indicates GPU speeds cited from original papers, and $\ddagger$ denotes CPU speeds. $\mp$ represents the quantitative results of using our proposed efficient training strategy illustrated in \textbf{Algorithm} \ref{alg:algorithm2}. The best results are highlighted in \textbf{Bold}.}
\begin{tabular}{c|c|cc|cc|c}
\toprule
Method & Pub.'Year & ODS-F & OIS-F & Throughput & Params. & Pre-training \\
\midrule
Human &  - & .803 & .803 & - & - & - \\
\midrule
HED \cite{HED} &  IJCV'17 & .788 & .808 & 49FPS & 14.7M & ImageNet \\
AMH-Net \cite{AMH} &   NeurIPS'17 & .798 & .829 & - & - & ImageNet \\
LPCB \cite{LPCB} &  ECCV'18 & .808 & .824 & 30FPS$\dagger$ & - & ImageNet \\
RCF \cite{RCF} &  PAMI'19 & .798 & .815 & 41FPS & 14.8M & ImageNet \\
BDCN \cite{BDCN_cvpr} &  CVPR'19 & .806 & .826 & 39FPS & 16.3M & ImageNet \\
DSCD \cite{DSCD} &  ACMMM'20 & .802 & .817 & - & - & ImageNet \\
LDC \cite{LDC} & ACMMM'21 & .799 & .816 & - & - & ImageNet \\
FCL \cite{FCL} &  NN'22 & .807 & .822 & - & - & ImageNet \\
EDTER \cite{EDTER} &  CVPR'22 & .824 & .839 & 5.6FPS & 468.84M & ImageNet \\
UAED \cite{treasure} &  CVPR'23 & .829 & .847 & 18.3FPS & 69M & ImageNet \\
DiffusionEdge \cite{Diffusionedge} & AAAI'24 & .834 & .848 & 1.3FPS & 224.9M & ImageNet \\
MuGE (M=3) \cite{MuGE} &  CVPR'24 & .845 & .854 & 10.7FPS & 69M & ImageNet \\
MuGE (M=11) \cite{MuGE} &  CVPR'24 & \textbf{.850} & \textbf{.856} & 3.2FPS & 69M & ImageNet \\
SAUGE \cite{SAUGE} &  AAAI'25 & .839 & .860 & - & - & SA-1B \\
\midrule
PEdger++ w/ VGG16 &  - & .838 & .849 &  87.5FPS &  4.7M & ImageNet \\
PEdger++ w/ ResNet50 &  - & .841 & .852 &  79FPS  & 7.8M & ImageNet \\
PEdger++Large w/ VGG16 &  - & .843 & .852 & 46.3FPS  & 8.3M & ImageNet \\
PEdger++Large w/ ResNet50 &  - & .847 & .856 & 37.6FPS  & 12.7M & ImageNet \\
\midrule
PEdger++ w/ VGG16$^{\mp}$ &  - & .834 & .846 &  87.5FPS &  4.7M & ImageNet \\
PEdger++ w/ ResNet50$^{\mp}$ &  - & .836 & .849 &  79FPS  & 7.8M & ImageNet \\
PEdger++Large w/ VGG16$^{\mp}$ &  - & .839 & .850 & 46.3FPS  & 8.3M & ImageNet \\
PEdger++Large w/ ResNet50$^{\mp}$ &  - & .843 & .853 & 37.6FPS  & 12.7M & ImageNet \\
\midrule
PEdger &  ACMMM'23 & .813 & .834 & 71FPS & 734K & No\\
PEdger-Large &  ACMMM'23 & .819 & .840 & 50FPS & 4.0M & No\\
\midrule
PEdger++ & - & .830 & .846 & \textbf{92FPS} & \textbf{716K} & No\\
PEdger++$^{\mp}$ &  - & .827 & .843 & \textbf{92FPS} & \textbf{716K} & No \\
PEdger++Large &   - & .841 & .852 & 48.5FPS & 4.3M & No\\
PEdger++Large$^{\mp}$ &   - & .836 & .848 & 48.5FPS & 4.3M & No\\
\bottomrule
\end{tabular}
\label{BSDS_result1}
\end{table*}

Following previous Bayesian deep learning methods \cite{Cheng_2019_CVPR,AdversarialBayesian,DBLP:conf/iclr/KimP0Y23,bayesian4,bit}, we draw support from Monte Carlo Dropout \cite{dropout,DropoutBayesian} to sample parameters across all layers. For each epoch, we sample a set of parameters ($S$ in total) from the momentum networks, say $\mathbf{\Theta}_s^R$ and  $\mathbf{\Theta}_s^{NR}$, $s \in \{1, ..., S \}$, and calculate their corresponding ensemble weights that maximize performance, through Eqs. \eqref{optimal_weight_R} and \eqref{optimal_weight_NR}. These sampled parameters and their ensemble weights remain fixed when training on $\mathcal{D}_t$ in the subsequent epoch. During testing, only the non-recurrent network is executed, whose parameters are finally derived by a weighted average of multiple sampled parameters for inference: 
\begin{equation}
    \mathbf{\Theta}^{NR} \defeq \sum_{s \defeq 1}^{S} \omega_s \cdot \mathbf{\Theta}^{NR}_s,
    \label{InferenceParams}
\end{equation}
where $\omega_s$ is calculated by $\|\mathbf{W}^{NR}_s\|_1 / \sum_{s\defeq1}^S \|\mathbf{W}^{NR}_s\|_1$. We can infer from the following theorem that, the calculated optimal weights in Eqs. \ref{optimal_weight_R} and \ref{optimal_weight_NR}, can also maximize the mutual information \cite{shannon1948communication,houlsby2011bayesian,lakshminarayanan2017simple,belghazi2018mine} between these ensemble weights and the ground truth edge maps, \textit{i.e.}, $\mathcal{I}(\mathbf{Y}_{nv}; \mathbf{W}^{R}_s|\mathbf{X}_{nv})$ and $\mathcal{I}(\mathbf{Y}_{nv}; \mathbf{W}^{NR}_s|\mathbf{X}_{nv})$, $s\in\{1, 2, ..., S\}$. It ensures that the ensemble weights are maximally informative about the ground truth, contributing to the improved generalization capacity and reduced epistemic uncertainty on unseen data. For the  simplicity, we unify the expression of $\mathbf{W}^{R}_s$ and $\mathbf{W}^{NR}_s$ as $\mathbf{W}_s$, and $\mathcal{G}^{R}_m$ and $\mathcal{G}^{NR}_m$ as $\mathcal{G}$ in the following theorem and its proof. 
\begin{theorem}[Mutual Information Maximization]
Given validation dataset $\mathcal{D}_v = \{(\mathbf{X}_{n_v}, \mathbf{Y}_{n_v})\}_{n_v=1}^{N_v}$, and multipled groups of sampled parameters $\{\boldsymbol{\Theta}_s\}_{s=1}^S$, $ s\in\{1, 2, ..., S\}$, the following optimization objectives are equivalent:
\begin{enumerate}
    \item Mutual Information Maximization:
    \[
    \max_{\mathbf{W}_s} \mathcal{I}(\mathbf{Y}_{n_v}; \mathbf{W_s} \mid \mathbf{X}_{n_v}),
    \]
    where $\mathcal{I}(\cdot, \cdot)$ represents the mutual information function.
    
    \item Validation Loss Minimization:
    \[
    \min_{\mathbf{W}_s} \frac{1}{N_v} \sum_{n_v=1}^{N_v} \mathcal{L}\left( \sum_{s=1}^S \mathbf{W}_s \circ \mathcal{G}(\mathbf{X}_{n_v}; \boldsymbol{\Theta}_s), \mathbf{Y}_{n_v} \right),
    \]
    with $\mathcal{L}$ as the binary cross-entropy loss and constraints $\sum_{s=1}^S \mathbf{W}_s = \mathbf{1}$, $\mathbf{W}_s \geq 0$.
\end{enumerate}
\label{theorem:equivalence}
\end{theorem}

\begin{proof}
The mutual information is expressed as:
\[
\mathcal{I}(\mathbf{Y}_{n_v};\mathbf{W}_s \mid \mathbf{X}_{n_v}) = H(\mathbf{Y}_{n_v} \mid \mathbf{X}_{n_v}) - H(\mathbf{Y}_{n_v} \mid \mathbf{W}_s, \mathbf{X}_{n_v}),
\]
where $H(\cdot)$ denotes entropy, and $s\in\{1, 2, ..., S\}$. Maximizing mutual information is equivalent to minimizing the conditional entropy $H(\mathbf{Y}_{n_v} \mid \mathbf{W}_s, \mathbf{X}_{n_v})$, which can be expanded as:
\[
H(\mathbf{Y}_{n_v} \mid \mathbf{W}_s, \mathbf{X}_{n_v}) = -\mathbb{E}_{p(\mathbf{Y},\mathbf{W},\mathbf{X})} \left[ \log p(\mathbf{Y} \mid \mathbf{W}, \mathbf{X}) \right].
\]

Under our ensemble method, the conditional distribution $p(\mathbf{Y} \mid \mathbf{W}, \mathbf{X})$ is parameterized by the fused prediction:
\[
\mathbf{M}_{n_v} = \sum_{s=1}^S \mathbf{W}_s \circ \mathcal{G}(\mathbf{X}_{n_v}; \boldsymbol{\Theta}_s),
\]
which gives:
\[
p(\mathbf{Y} \mid \mathbf{W}, \mathbf{X}) = \prod_{i,j} \left( \mathbf{M}_{i,j} \right)^{\mathbf{Y}_{i,j}} \left(1 - \mathbf{M}_{i,j}\right)^{1 - \mathbf{Y}_{i,j}},
\]
where $(i,j)$ indexes pixel locations placing at the $i$-the row and $j$-th column.

The conditional entropy then becomes:
\begin{align*}
&H(\mathbf{Y}_{n_v} \mid \mathbf{W}_s, \mathbf{X}_{n_v}) =\\
&-\mathbb{E} \left[ \sum_{i,j} \left(\mathbf{Y}_{i,j} \log \mathbf{M}_{i,j} + (1 - \mathbf{Y}_{i,j}) \log (1 - \mathbf{M}_{i,j}) \right) \right] \\
&= \mathbb{E} \left[ \mathcal{L}(\mathbf{M}_{n_v}, \mathbf{Y}_{n_v}) \right],
\end{align*}
Minimizing this expectation over the true data distribution is equivalent to minimizing the empirical loss on the validation set:
\[
\min_{\mathbf{W}_s} \frac{1}{N_v} \sum_{n_v=1}^{N_v} \mathcal{L} \left( \sum_{s=1}^S \mathbf{W}_s \circ \mathcal{G}(\mathbf{X}_{n_v}; \boldsymbol{\Theta}_s), \mathbf{Y}_{n_v} \right).
\]
The constraints $\sum_{s=1}^S \mathbf{W}_s = \mathbf{1}$ and $\mathbf{W}_s \geq 0$ ensure that $\mathbf{M}_{n_v}$ forms a valid convex combination of predictions. 

\end{proof}

\begin{table*}[th!]
\centering
\caption{Performance comparison on the augmented BSDS data and extra PASCAL VOC data for training. Running speeds are benchmarked on an RTX 3090 GPU with input images sized 320 $\times$ 480 images unless stated otherwise. $\dagger$ indicates GPU speeds cited from original papers, and $\ddagger$ denotes CPU speeds. $\mp$ represents the quantitative results of using our proposed efficient training strategy illustrated in \textbf{Algorithm} \ref{alg:algorithm2}. The best results are highlighted in \textbf{Bold}.}
\begin{tabular}{c|c|cc|cc|c}
\toprule
Method & Pub.'Year & ODS-F & OIS-F & Throughput & Params. & Pre-training \\
\midrule
Human &  - & .803 & .803 & - & - & - \\
\midrule
CEDN \cite{CEDN} &  CVPR'16 & .788 & .804 & 10FPS$\dagger$ & - & ImageNet \\
CED \cite{CED} &   CVPR'17 & .794 & .811 & - & 14.9M & ImageNet \\
LPCB \cite{LPCB} &   ECCV'18 & .815 & .834 & 30FPS$\dagger$ & - & ImageNet \\
RCF \cite{RCF} &  PAMI'19 & .806 & .823 & 41FPS & 14.8M & ImageNet \\
BDCN \cite{BDCN_cvpr} &  CVPR'19 & .820 & .838 & 39FPS & 16.3M & ImageNet \\
DSCD \cite{DSCD} &  ACMMM'20 & .813 & .836 & - & - & ImageNet \\
LDC \cite{LDC} & ACMMM'21 & .812 & .826 & - & - & ImageNet \\
FCL \cite{FCL} &  NN'22 & .815 & .834 & - & - & ImageNet \\
EDTER \cite{EDTER} &  CVPR'22 & .832 & .847 & 5.6FPS & 468.84M & ImageNet \\
UAED \cite{treasure} &  CVPR'23 & .838 & .855 & 18.3FPS & 69M & ImageNet \\
MuGE (M=3) \cite{MuGE} &  CVPR'24 & .852 & .859 & 10.7FPS & 69M & ImageNet \\
MuGE (M=11) \cite{MuGE} &  CVPR'24 & \textbf{.855} & \textbf{.860} & 3.2FPS & 69M & ImageNet \\
SAUGE \cite{SAUGE} &  AAAI'25 & .842 & .862 & - & - & SA-1B \\
\midrule
PEdger++ w/ VGG16 &  - & .843 & .854 & 87.5FPS &  4.7M & ImageNet \\
PEdger++ w/ ResNet50  &  - & .846 & .856 & 79FPS &  7.8M & ImageNet \\
PEdger++Large w/ VGG16 &  - & .848 & .855 & 46.3FPS &  8.3M & ImageNet \\
PEdger++Large w/ ResNet50 &  - & .857 & .861 & 37.6FPS &  12.7M & ImageNet \\
\midrule
PEdger++ w/ VGG16$^{\mp}$ &  - & .840 & .851 & 87.5FPS &  4.7M & ImageNet \\
PEdger++ w/ ResNet50$^{\mp}$  &  - & .842 & .854 & 79FPS &  7.8M & ImageNet \\
PEdger++Large w/ VGG16$^{\mp}$ &  - & .843 & .853 & 46.3FPS &  8.3M & ImageNet \\
PEdger++Large w/ ResNet50$^{\mp}$ &  - & .852 & .856 & 37.6FPS &  12.7M & ImageNet \\
\midrule
ISCRA \cite{ISCRA} &  CVPR'13 & .717 & .752 & - & - & No\\
SE \cite{se} &  ICCV'13 & .743 & .763 &  12.5FPS$\ddagger$ & - & No \\
OEF \cite{oef} &  CVPR'15 & .746 & .770 &  2/3FPS$\ddagger$ & - & No \\
\midrule
PiDiNet \cite{PiDiNet_TPAMI} &   PAMI'23 & .807 & .823 & 76FPS & 710K & No \\
PEdger &  ACMMM'23 & .821 & .840 & 71FPS & 734K & No\\
PEdger-Large &  ACMMM'23 & .839 & .844 & 50FPS & 4.0M & No\\
\midrule
PEdger++ &  - & .835 & .850 & \textbf{92FPS} & \textbf{716K} & No \\
PEdger++$^{\mp}$ &  - & .832 & .845 & \textbf{92FPS} & \textbf{716K} & No \\
PEdger++Large &  - & .848 & .859 & 48.5FPS & 4.3M & No \\
PEdger++Large$^{\mp}$ &  - & .844 & .856 & 48.5FPS & 4.3M & No \\
\bottomrule
\end{tabular}
\label{BSDS_result2}
\end{table*}

\begin{table*}[t]
\centering
\small
\caption{Performance comparison with multi-scale testing. }
\begin{tabular}{c|c|cc|cc|c}
  \toprule
Method & Dataset & ODS-F & OIS-F & Params. & Throughput & Pre-training \\
\midrule
FCL-MS &  & .816 & .833 &  - & - & ImageNet \\
EDTER-MS &    & .840 & .858 &  468.84M & 1.7FPS & ImageNet \\
UAED-MS &  & .837 & .855 &  69M & 7.3FPS & ImageNet \\
MuGE (M=3)-MS &  & .853 & .863 &  69M & 4.5FPS & ImageNet \\
MuGE (M=11)-MS & BSDS & .858 & .864 &  69M & 1.2FPS & ImageNet \\
\cmidrule(r){1-1}\cmidrule(r){3-7}
PEdger++ w/ VGG16-MS &  w/o VOC  & .844 & .856 & 4.7M & 30.5FPS & ImageNet \\
PEdger++ w/ ResNet50-MS &  & .847 & .856 & 7.8M & 25FPS & ImageNet \\
PEdger++Large w/ VGG16-MS &  & .849 & .857 & 9.3M & 14.7FPS & ImageNet \\
PEdger++Large w/ ResNet50-MS & & .853 & .861 & 12.7M & 11.6FPS & ImageNet \\
\cmidrule(r){1-1}\cmidrule(r){3-7}
PEdger-MS &  & .819 & .840 &  734K & 25.6FPS & No \\
PEdger++-MS & & .837 & .854 &  716K & 32.5FPS & No \\
\midrule
RCF-MS &  & .811 & .830 &  14.8M & 13.6FPS & ImageNet \\
BDCN-MS &  & .828 & .844 &  16.3M & 12FPS & ImageNet \\
FCL-MS &  & .826 & .845 &  - & - & ImageNet \\
LDC-MS &  & .819 & .834 &  - & - & ImageNet \\
EDTER-MS & & .848 & .865 &  468.84M & 1.7FPS & ImageNet \\
UAED-MS &  & .844 & .864 &  69M & 7.3FPS & ImageNet \\
MuGE (M=3)-MS &  & .858 & .865 &  69M & 4.5FPS & ImageNet \\
MuGE (M=11)-MS & BSDS & .861 & .867 &  69M & 1.2FPS & ImageNet \\
\cmidrule(r){1-1}\cmidrule(r){3-7}
PEdger++ w/ VGG16-MS &  w/ VOC  & .848 & .860 & 4.7M & 30.5FPS & ImageNet \\
PEdger++ w/ ResNet50-MS &  & .850 & .862 & 7.8M & 25FPS & ImageNet \\
PEdger++Large w/ VGG16-MS &  & .853 & .859 & 8.3M & 14.7FPS & ImageNet \\
PEdger++Large w/ ResNet50-MS &  & .862 & .871 & 12.7M & 11.6FPS & ImageNet \\
\cmidrule(r){1-1}\cmidrule(r){3-7}
PEdger-MS &  & .825 & .844 &  734K & 25.6FPS & No \\
PEdger++-MS & & .843 & .858 &  716K & 32.5FPS & No \\
\bottomrule
\end{tabular}
\label{BSDS_results_MS}
\end{table*}

\begin{table}[t]
\centering
\small
\caption{Comparison of decreases in total training time and detection accuracy. BSDS w/o VOC means only using BSDS data for training, while BSDS w/ VOC introduces extra PASCAL VOC data for training.}
\resizebox{\linewidth}{!}{\begin{tabular}{c|ccc}
  \toprule
Method & $\Delta$ ODS-F & $\Delta$ OIS-F & $\Delta$ Training Time \\
\midrule
\multicolumn{4}{c}{BSDS w/o VOC} \\
\midrule
PEdger++ w/ VGG16$^{\mp}$ &  -0.4\% & -0.3\% & -36.8\%    \\
PEdger++ w/ ResNet50$^{\mp}$  &   -0.5\% & -0.4\% & -34.9\%    \\
PEdger++Large w/ VGG16$^{\mp}$ &   -0.5\% & -0.4\% & -35.7\%  \\
PEdger++Large w/ ResNet50$^{\mp}$ &   -0.5\% & -0.3\%  & -32.1\%  \\
PEdger++$^{\mp}$ &   -0.4\%  &  -0.3\% & -37.6\%  \\
PEdger++Large$^{\mp}$ &   -0.4\%  &  -0.3\% & -36.4\%  \\
\midrule
\multicolumn{4}{c}{BSDS w/ VOC} \\
\midrule
PEdger++ w/ VGG16$^{\mp}$ &    -0.4\% &  -0.3\% & -42.5\%  \\
PEdger++ w/ ResNet50$^{\mp}$  &    -0.3\% &  -0.2\% & -40.8\%  \\
PEdger++Large w/ VGG16$^{\mp}$ &    -0.6\% &  -0.2\% & -41.2\%  \\
PEdger++Large w/ ResNet50$^{\mp}$ &    -0.5\% &  -0.5\% & -39.8\%  \\
PEdger++$^{\mp}$ &   -0.4\%  & -0.5\%  & -45.3\% \\
PEdger++Large$^{\mp}$ &   -0.4\%  & -0.5\%  & -43.9\% \\
\bottomrule
\end{tabular}}
\label{rate}
\end{table}

\begin{table}[t]
\centering
\caption{Comparison on the NYUD dataset. All results are computed with single-scale RGB inputs. The running speeds are tested on an RTX 3090 GPU with 500 $\times$ 500 inputs. $\mp$ represents the quantitative results of using our proposed efficient training strategy illustrated in \textbf{Algorithm} \ref{alg:algorithm2}. Best results are \textbf{Bolded}.}
\resizebox{0.48\textwidth}{!}{
\begin{tabular}{c|cc|c}
  \toprule
 Methods & ODS-F & OIS-F & Throughput \\
\midrule
 gPb-UCM \cite{gpu-ucm}  & .632 & .661 & 1/360FPS$\dagger$\\
 SE \cite{se} & .695 & .708 & 5FPS$\dagger$ \\
  OEF \cite{oef}  & .651 & .667 & -\\
\midrule
HED \cite{HED}  & .720 & .734 & 45FPS\\
LPCB \cite{LPCB}  & .739 & .754 & - \\
RCF \cite{RCF}  & .743 & .757 & 41FPS \\
AMH-Net \cite{AMH}  & .744 & .758 & - \\
BDCN \cite{BDCN_cvpr}  & .748 & .763 & 39FPS \\
EDTER \cite{EDTER}  & .774 & .789 & 2.1FPS \\
DiffusionEdge \cite{Diffusionedge}  & .761 & .766 & 0.8FPS \\
\midrule
PiDiNet \cite{PiDiNet_TPAMI}  & .733 & .747 & 72FPS \\
PEdger &  .742 & .757 & 70FPS \\
\midrule
PEdger++ & .765 & .769 & \textbf{88FPS} \\
PEdger++ w/ VGG16  & .775 & .778 & 73FPS \\
PEdger++ w/ ResNet50 & .778 & .780 & 69FPS \\
PEdger++Large w/ VGG16  & .780 & .784 & 65FPS \\
PEdger++Large w/ ResNet50  & \textbf{.782} & \textbf{.785} & 57FPS \\
\midrule
PEdger++$^{\mp}$ & .759 & .764 & \textbf{88FPS} \\
PEdger++ w/ VGG16$^{\mp}$  & .769 & .773 & 73FPS \\
PEdger++ w/ ResNet50$^{\mp}$ & .772 & .777 & 69FPS \\
PEdger++Large w/ VGG16$^{\mp}$  & .774 & .778 & 65FPS \\
PEdger++Large w/ ResNet50$^{\mp}$  & .776 & .780 & 57FPS \\
\bottomrule
\end{tabular}}
\label{NYUD}
\end{table}

\subsection{Robust Collaborative Learning}
To achieve a balance between accuracy and efficiency, we introduce a novel collaborative learning algorithm to investigate the robust knowledge across diverse network architectures, varying training moments, and multiple parameter samplings. Through this approach, we  deliver a solution that excels in accuracy, speed, and compact model size, the whole process of which is summarized  in \textbf{Algorithm} \ref{alg:algorithm}.


\begin{figure}[t]
    \centering
    \includegraphics[width=0.9\linewidth]{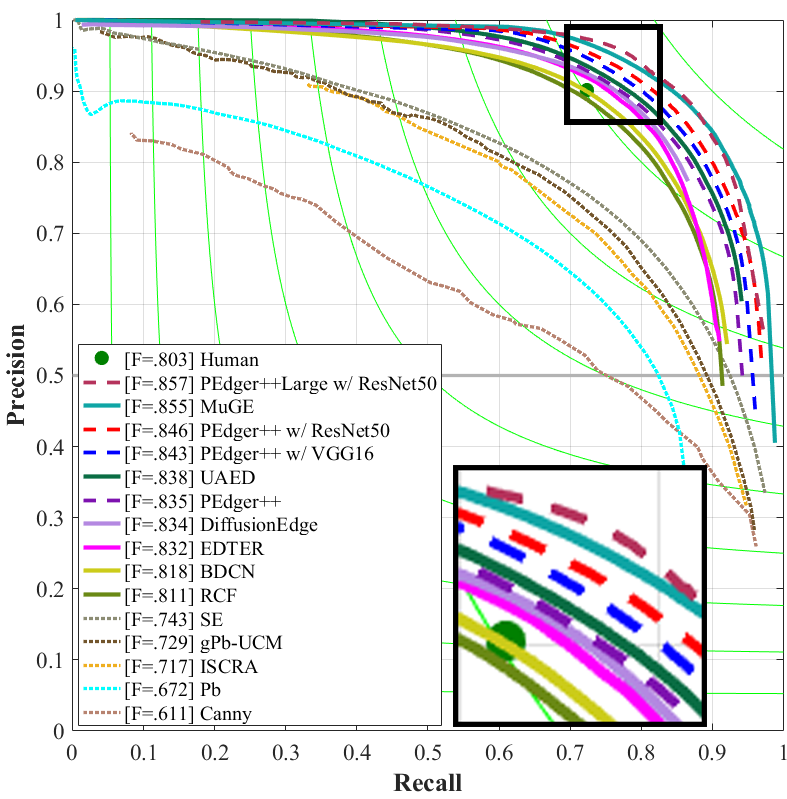}
    \caption{Precision-Recall curves on the BSDS dataset. Our results, along with those of other deep learning and non-deep learning competitors, are represented by dashed, solid, dotted lines, respectively.}
    \label{pr}
\end{figure}

In the initial epoch ($j \defeq 0$), the networks $\mathcal{G}^R$ and $\mathcal{G}^{NR}$ are trained solely on the original ground truth labels $\mathbf{Y}_{n_t}$. Following this first epoch, the parameters of the momentum networks, $\mathcal{G}_m^R$ and $\mathcal{G}_m^{NR}$, are updated by directly copying the trained parameters from $\mathcal{G}^R$ and $\mathcal{G}^{NR}$. Beginning in the second epoch and continuing thereafter, the predictions from the momentum networks are incorporated into training. Specifically, the predictions from the momentum networks across multiple sampled parameters are aggregated using the optimized ensemble weights $\mathbf{W}_s^{R}$ and $\mathbf{W}_s^{NR}$ ($s \in \{1, ..., S\}$), to produce $\mathbf{M}^{R}_{n_t}$ and $\mathbf{M}^{NR}_{n_t}$:
\begin{equation}
\begin{aligned}
\mathbf{M}_{n_t}^R &\defeq \sum_{s\defeq 1}^{S} \mathbf{W}_s^{R} \circ  \mathcal{G}_m^R(\mathbf{X}_{n_t};\mathbf{\Theta}^R_s), 
\\
\mathbf{M}_{n_t}^{NR} &\defeq \sum_{s\defeq 1}^{S} \mathbf{W}_s^{NR} \circ \mathcal{G}_m^{NR}(\mathbf{X}_{n_t}; \mathbf{\Theta}^{NR}_s).
\end{aligned}
\label{eq:Mnt}
\end{equation}
 To further integrate the knowledge from both network structures, we fuse the predictions from the two momentum networks based on their pixel-wise confidences by:
\begin{equation}
\begin{aligned}
    \mathbf{M}_{n_t} &\defeq \frac{\mathbf{M}^{R}_{n_t} \circ|\mathbf{M}^{R}_{n_t} - 0.5|+\mathbf{M}^{NR}_{n_t} \circ|\mathbf{M}^{NR}_{n_t} - 0.5|}{|\mathbf{M}^{R}_{n_t} - 0.5| + |\mathbf{M}^{NR}_{n_t} - 0.5|},
    \label{uncertainty}
\end{aligned}
\end{equation}
where $|\cdot|$ is pixel-wise absolute operation, and the division is also pixel-wise. When a network's prediction is closer to 0.5 compared to the other network's prediction, it is deemed less confident, indicating greater uncertainty in its output. Consequently, this network should contribute less to the fusion process, and \textit{vice versa}. 

Having the estimated $\mathbf{M}_{n_t}$ integrating knowledge across training moments, network structures, and parameter samples, the original target $\mathbf{\mathbf{Y}}_{n_t}$ can be replaced with a refined  version $\tilde{\mathbf{Y}}_{n_t}$, defined as:
\begin{equation}
    \tilde{\mathbf{Y}}_{n_t} \defeq \eta_{j} \cdot \mathbf{M}_{n_t}  + (1 - \eta_{j}) \cdot \mathbf{\mathbf{Y}}_{n_t},
    \label{soft}
\end{equation}
where $\eta_{j}$ is a weight controlling the reliance on $\mathbf{M}_{n_t}$, at the $j$-th epoch. It is worth to mention that, $\tilde{\mathbf{Y}}_{n_t}$ functions as a soft target \cite{DBLP:conf/cvpr/DiazM19,DBLP:conf/iccv/ZiZMJ21,teachingwithsoft,DBLP:conf/iclr/ZhouSCZWYZ21,DBLP:conf/ijcai/WangNFHWWLZL22,DBLP:conf/aaai/LienenH21}, which tempers overconfident class probabilities inherent in the original hard target. In our collaborative setting, since all the models are optimized collaboratively from scratch, they generally lack sufficient knowledge of the data distributions in the early training stages. To address this issue, and inspired by the self-distillation technique \cite{SelfDistillation}, we gradually increase $\eta_{j}$ for adjusting the importance of the $j$-th epoch over the course of training, which is simply computed by:
\begin{equation}
    \eta_{j} \defeq \eta_{J} \cdot \frac{j}{J},
    \label{weightincrease}
\end{equation}
where $J$ is the total number of training epochs. Empirically setting $\eta_{J} \defeq 0.8$ works sufficiently well in our experiments.

Furthermore, to accelerate the training process, we provide an efficient version of the collaborative training algorithm. To be more specific, we abandon the recurrent network during training, retaining only a single non-recurrent network. To integrate the information from heterogeneous models, we apply multiple random weight pruning operations to this non-recurrent network. The network parameters deemed less useful are pruned with a predefined probability, thereby generating diverse network architectures. The heterogeneity of network architectures is no longer achieved by incorporating distinct recurrent and non-recurrent models. Instead, it is realized through multiple random pruning iterations on the single non-recurrent network. Besides, the random pruning process can also be interpreted as a form of Bayesian sampling, where the pruned parameters represent the sampling outcome. Please refer to \textbf{Algorithm} \ref{alg:algorithm2} for details, which reduces the training time by approximately 40\% compared to the original collaborative learning strategy in \textbf{Algorithm} \ref{alg:algorithm}.

\subsection{Loss Function} 
Our models are trained on pairs of (input image, ground truth) from the newly constructed training set $\mathcal{D}_t$ split from $\mathcal{D}$. For the training, our penalty is derived from the widely-used cross-entropy loss, and applied to all side-outputs and final edge maps, which is written as:
\begin{equation}
\begin{aligned}
\mathcal{L}(\mathcal{G}(\mathbf{X}_{n_t}), \tilde{\mathbf{Y}}_{n_t}) \defeq &\alpha \cdot\| \tilde{\mathbf{Y}}_{n_t}\circ \log(\mathcal{G}(\mathbf{X}_{n_t}))\|_1+ \\
                \beta\cdot\|& (\mathbf{1} - \tilde{\mathbf{Y}}_{n_t})\circ \log(\mathbf{1} - \mathcal{G}(\mathbf{X}_{n_t})\|_1.
\end{aligned}
\label{BCE}
\end{equation}
In this work, $\alpha$ and $\beta$ are determined via:
\begin{equation}
\begin{aligned}
\alpha &\defeq \frac{\lambda\cdot\|\mathbf{\mathbf{Y}}_{n_t}\circ \tilde{\mathbf{Y}}_{n_t}\|_1}{\|\mathbf{\mathbf{Y}}_{n_t}\circ \tilde{\mathbf{Y}}_{n_t}\|_1+\|(\mathbf{1}-\mathbf{\mathbf{Y}}_{n_t})\circ (\mathbf{1}-\tilde{\mathbf{Y}}_{n_t})\|_1},\\
\beta &\defeq \frac{\|(\mathbf{1}-\mathbf{\mathbf{Y}}_{n_t})\circ (\mathbf{1}-\tilde{\mathbf{Y}}_{n_t})\|_1}{\|\mathbf{\mathbf{Y}}_{n_t}\circ \tilde{\mathbf{Y}}_{n_t}\|_1+\|(\mathbf{1}-\mathbf{\mathbf{Y}}_{n_t})\circ (\mathbf{1}-\tilde{\mathbf{Y}}_{n_t})\|_1},
\end{aligned}
\label{eq:ab}
\end{equation} 
where $\lambda$ is a hyper-parameter for balancing these two terms. The final learning objective contains $\mathcal{L}^R$ and $\mathcal{L}^{NR}$ respectively for $\mathcal{G}^R$ and $\mathcal{G}^{NR}$ as:
\begin{equation}
\begin{aligned}
\mathcal{L}^R\defeq& \sum_{t=1}^T \mathcal{L}(\sigma(\mathbf{F}_{n_t}^{f2c(t)}), \tilde{\mathbf{Y}}_{n_t}) +\\
\sum_{t=1}^T\mathcal{L}(&\sigma(\mathbf{F}_{n_t}^{c2f(t)}), \tilde{\mathbf{Y}}_{n_t}) 
+\mathcal{L}(\mathcal{G}^R(\mathbf{X}_{n_t}), \tilde{\mathbf{Y}}_{n_t}),
\end{aligned}
\label{loss1}
\end{equation}
\begin{equation}
\begin{aligned}
\mathcal{L}^{NR}\defeq& \sum_{t=1}^T \mathcal{L}(\sigma(\tilde{\mathbf{F}}_{n_t}^{f2c(t)}), \tilde{\mathbf{Y}}_{n_t}) +\\ \sum_{t=1}^T\mathcal{L}(&\sigma(\tilde{\mathbf{F}}_{n_t}^{c2f(t)}), \tilde{\mathbf{Y}}_{n_t}) 
+\mathcal{L}(\mathcal{G}^{NR}(\mathbf{X}_{n_t}), \tilde{\mathbf{Y}}_{n_t}),
\end{aligned}
\label{loss2}
\end{equation}
where $\sigma(\cdot)$ represents the Sigmoid function. Please notice that, for the recurrent network, $T$ represents the total number of recurrences, while for the non-recurrent one, $T$ is the amount of different scales, as shown in Fig. \ref{arch}. 


\begin{figure*}[t]
    \centering
    \includegraphics[width=0.19\linewidth, frame]{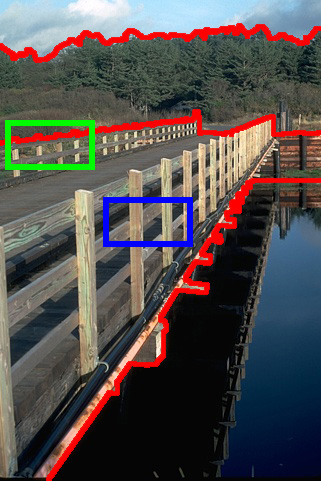}
    \includegraphics[width=0.19\linewidth, frame]{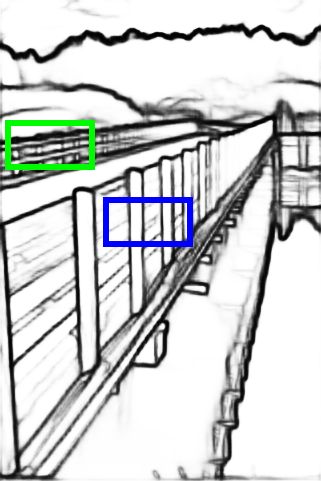}
    \includegraphics[width=0.19\linewidth, frame]{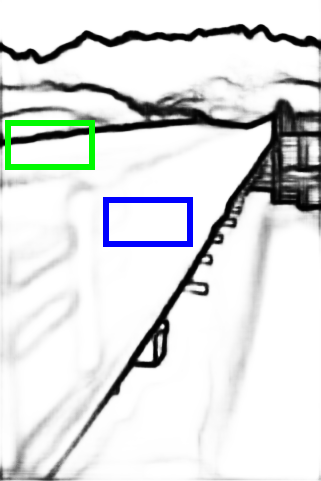}
    \includegraphics[width=0.19\linewidth, frame]{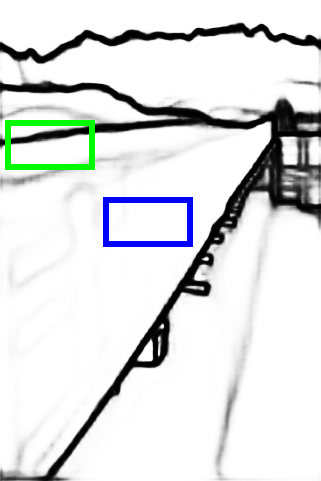}
    \includegraphics[width=0.19\linewidth, frame]{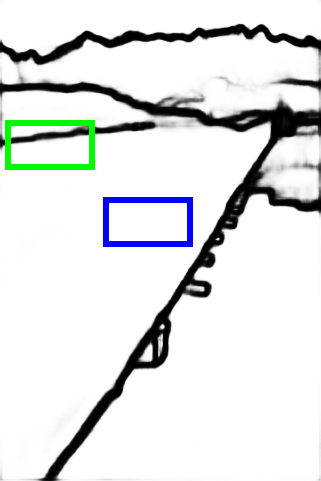}

    \includegraphics[width=0.192\linewidth]{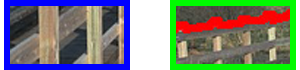}
    \includegraphics[width=0.192\linewidth]{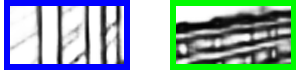}
    \includegraphics[width=0.192\linewidth]{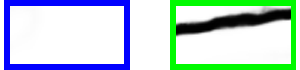}
    \includegraphics[width=0.192\linewidth]{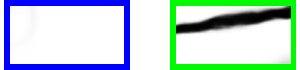}
    \includegraphics[width=0.192\linewidth]{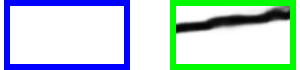}

    \includegraphics[width=0.19\linewidth, frame]{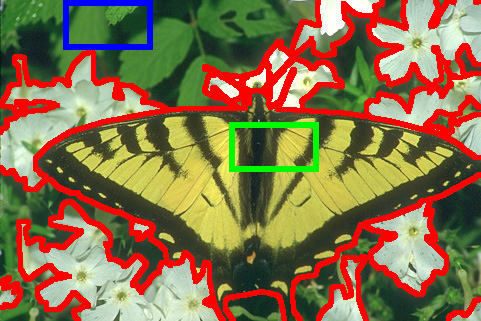}
    \includegraphics[width=0.19\linewidth, frame]{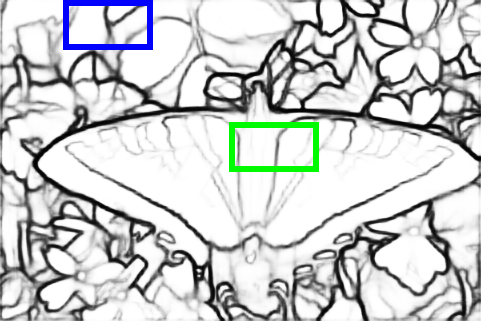}
    \includegraphics[width=0.19\linewidth, frame]{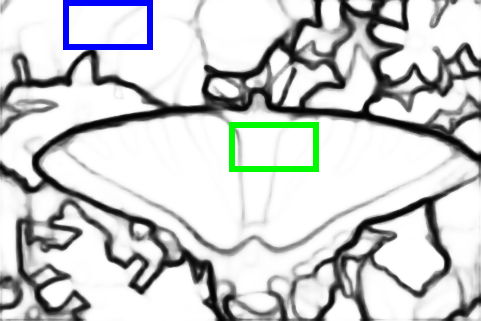}
    \includegraphics[width=0.19\linewidth, frame]{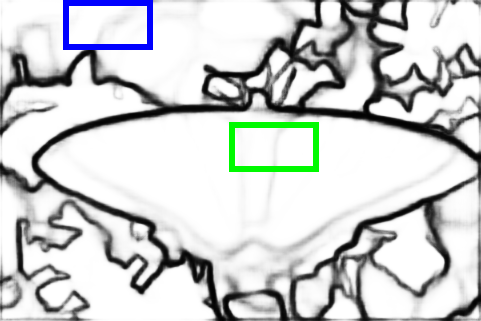}
    \includegraphics[width=0.19\linewidth, frame]{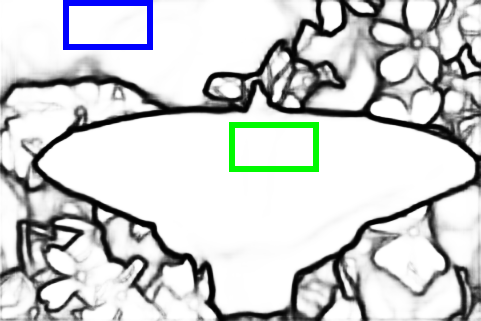}

    \includegraphics[width=0.192\linewidth]{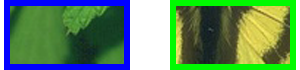}
    \includegraphics[width=0.192\linewidth]{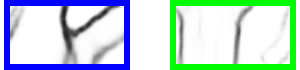}
    \includegraphics[width=0.192\linewidth]{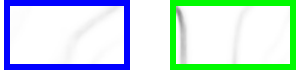}
    \includegraphics[width=0.192\linewidth]{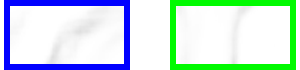}
    \includegraphics[width=0.192\linewidth]{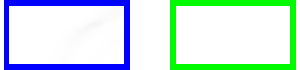}

    \includegraphics[width=0.19\linewidth, frame]{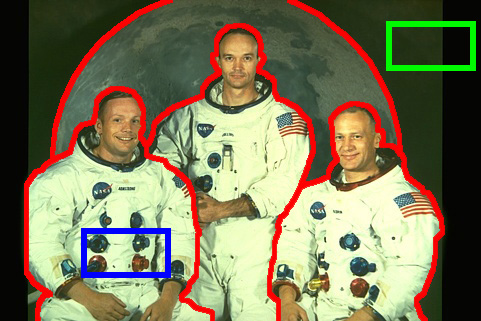}
    \includegraphics[width=0.19\linewidth, frame]{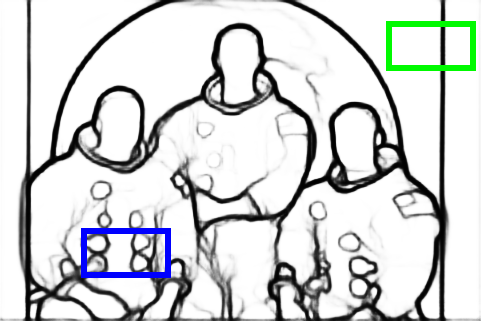}
    \includegraphics[width=0.19\linewidth, frame]{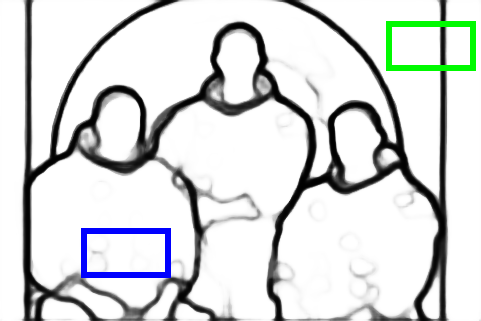}
    \includegraphics[width=0.19\linewidth, frame]{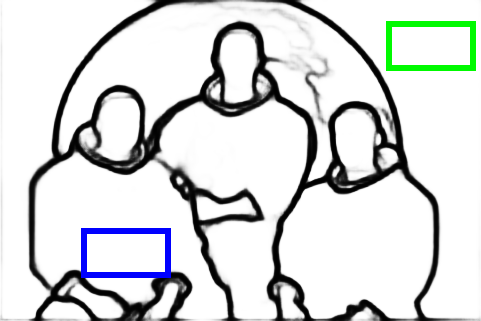}
    \includegraphics[width=0.19\linewidth, frame]{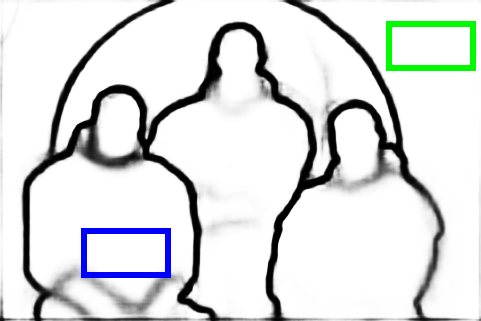}

    \includegraphics[width=0.192\linewidth]{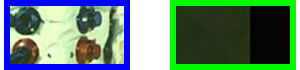}
    \includegraphics[width=0.192\linewidth]{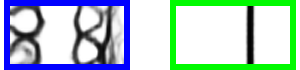}
    \includegraphics[width=0.192\linewidth]{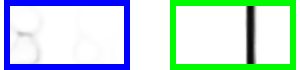}
    \includegraphics[width=0.192\linewidth]{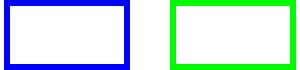}
    \includegraphics[width=0.192\linewidth]{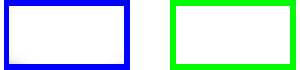}

    \quad\quad\; \small{Input\&GT} \quad\quad\quad\quad\quad\quad\quad   \small{PEdger} \quad\quad\quad\quad\quad\quad\quad\;   \small{PEdger++} \quad\quad\quad\;   \small{PEdger++ w/ VGG16} \quad \small{PEdger++ w/ ResNet50}
    
    \caption{Visual comparisons between our models with and without pre-training.}
    \label{com_pretrainandnopretrain}
\end{figure*}

\begin{table*}[t]
\centering
\small
\caption{Comparison on the Multicue dataset. All results are computed with single-scale inputs. $\mp$ represents the quantitative results of using our proposed efficient training strategy illustrated in \textbf{Algorithm} \ref{alg:algorithm2}. Best result are \textbf{Bolded}.}
\begin{tabular}{c|cc|cc}
  \toprule
 \multirow{2}{*}{Methods} & \multicolumn{2}{c|}{Edge} & \multicolumn{2}{c}{Boundary} \\
 & ODS-F & OIS-F & ODS-F & OIS-F \\
 \midrule
 Human & .750 (.024) & - & .760 (.017) & -\\
\midrule
Multicue & .830 (.002) & - & .720 (.014) & -\\
 HED \cite{HED} & .851 (.014) & .864 (.011) & .814 (.011) & .822 (.008) \\ 
 RCF \cite{RCF} & .857 (.004) & .862 (.004) & .817 (.004) & .825 (.005) \\
 BDCN \cite{BDCN_cvpr} & .891 (.001) & .898 (.002) & .836 (.001) & .846 (.003) \\
 DSCD \cite{DSCD} & .871 (.007) & .876 (.002) & .828 (.003) & .835 (.004) \\
 EDTER \cite{EDTER} & .894 (.005) & .900 (.003) & .861 (.003) & .870 (.004) \\
 DiffusionEdge \cite{Diffusionedge} & .904 (-) & .909 (-) & - & - \\
 UAED \cite{treasure} & .895 (.002) & .902 (.001) & .864 (.004) & .872 (.006) \\
 MuGE \cite{MuGE} & .898 (.004) & .900 (.004) & .875 (.006) & .879 (.006) \\
 \midrule
  PiDiNet \cite{PiDiNet_TPAMI} & .855(.007) & .860(.005) & .818(.003) & .830(.005) \\
  PEdger & .883(.008) & .892(.004) & .828(.005) & .839(.006) \\
\midrule
  PEdger++ & .901(.005) & .903(.006) & .866(.005) & .874(.006) \\
  PEdger++ w/ VGG16 & .905(.003) & .910(.004) & .873(.006) & .879(.004) \\
  PEdger++ w/ ResNet50 & .907(.006) & .913(.002) & .876(.003) & .881(.006) \\
  PEdger++Large w/ VGG16 & .908(.005) & .915(.007) & .875(.001) & .882(.005) \\
  PEdger++Large w/ ResNet50 & \textbf{.911(.006)} & \textbf{.916(.004)} & \textbf{.878(.002)} & \textbf{.884(.003)} \\
\midrule
  PEdger++$^{\mp}$ & .898(.003) & .901(.007) & .862(.006) & .869(.004) \\
  PEdger++ w/ VGG16$^{\mp}$ & .901(.004) & .905(.003) & .869(.007) & .874(.005) \\
  PEdger++ w/ ResNet50$^{\mp}$ & .902(.004) & .908(.005) & .871(.007) & .875(.003) \\
  PEdger++Large w/ VGG16$^{\mp}$ & .903(.003) & .910(.005) & .871(.004) & .878(.007) \\
  PEdger++Large w/ ResNet50$^{\mp}$ & .907(.005) & .911(.006) & .873(.004) & .879(.005) \\
\bottomrule
\end{tabular}
\label{multicue}
\end{table*}

\begin{table}[t]
\centering
\small
\caption{Quantitative results of different model sizes through adjusting the number of channels and the depth of network layers. All results are with single-scale inputs. }
\resizebox{0.48\textwidth}{!}{
\begin{tabular}{c|ccc|ccc}
  \toprule
 \multirow{2}{*}{Setting} & \multicolumn{3}{c|}{PEdger} & \multicolumn{3}{c}{PEdger++} \\
 & ODS-F & OIS-F & Params. & ODS-F & OIS-F & Params. \\
 \midrule
\multicolumn{7}{c}{BSDS w/o VOC} \\
 \midrule
 Tiny & .801 & .822 & 317K & .808 & .833 & 315K \\
 Small & .807 & .828 & 496K & .814 & .841 & 487K \\
 Normal & .813 & .834 & 734K & .830 & .846 & 716K\\
 Large & .819 & .840 & 4.0M & .841 & .852 & 4.3M \\
  \midrule
  \multicolumn{7}{c}{BSDS w/ VOC} \\
 \midrule
 Tiny & .810 & .826 & 317K & .819 & .837 & 315K \\
 Small & .815 & .832 & 496K & .828 & .844 & 487K \\
 Normal & .821 & .840 & 734K & .835 & .850 & 716K\\
 Large & .839 & .844 & 4.0M & .848 & .859 & 4.3M \\
  \midrule
\multicolumn{7}{c}{NYUD} \\
 \midrule
 Tiny & .725 & .736 & 317K & .737 & .752 & 315K \\
 Small & .736 & .753 & 496K & .758 & .761 & 487K \\
 Normal & .742 & .757 & 734K & .765 & .769 & 716K\\
 Large & .748 & .762 & 4.0M & .771 & .776 & 4.3M \\
  \midrule
\multicolumn{7}{c}{Multicue} \\
 \midrule
 Tiny & .858 & .871 & 317K & .877 & .884 & 315K \\
 Small & .875 & .886 & 496K & .890 & .897 & 487K \\
 Normal & .883 & .892 & 734K & .901 & .903 & 716K \\
 Large & .891 & .899 & 4.0M & .909 & .913 & 4.3M \\
\bottomrule
\end{tabular}}
\label{scalability}
\end{table}

\begin{figure*}[tp]
    \centering

    \includegraphics[width=0.143\linewidth,frame]{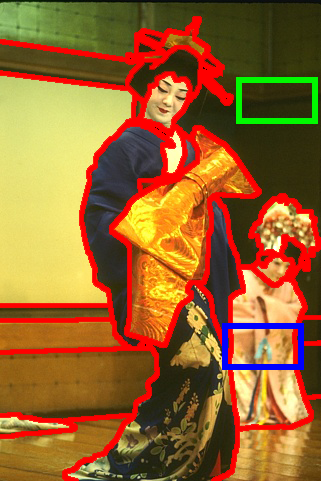}
    \includegraphics[width=0.143\linewidth,frame]{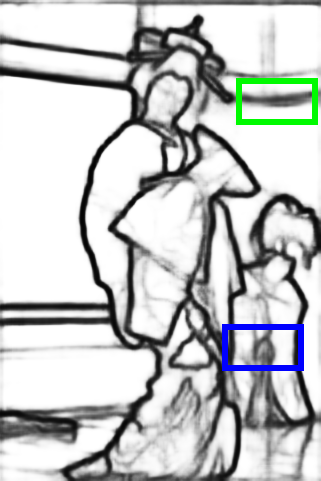}
    \includegraphics[width=0.143\linewidth,frame]{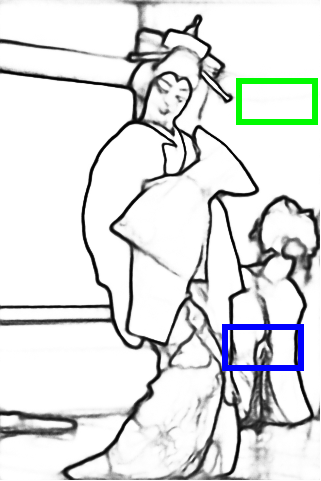}
    \includegraphics[width=0.143\linewidth,frame]{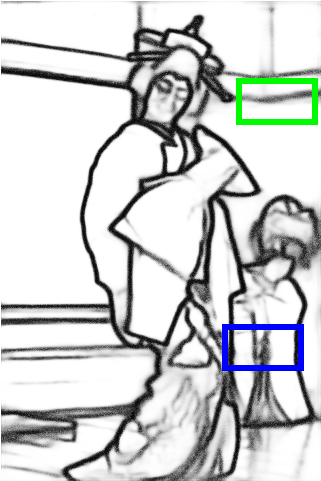}    
    \includegraphics[width=0.143\linewidth,frame]{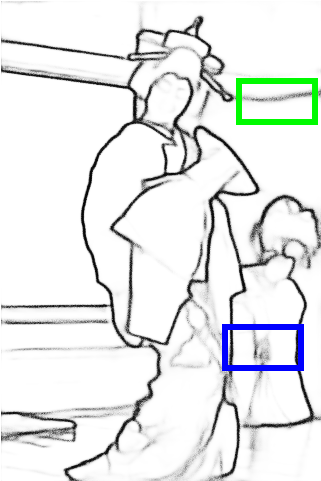}
    \includegraphics[width=0.143\linewidth,frame]{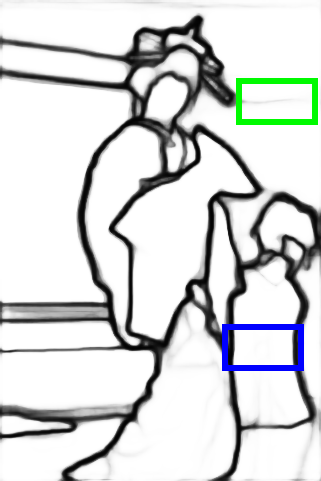}
\\ 

    \includegraphics[width=0.145\linewidth]{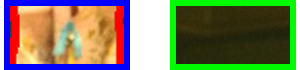}
    \includegraphics[width=0.145\linewidth]{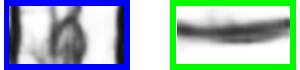}
    \includegraphics[width=0.145\linewidth]{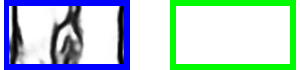}
    \includegraphics[width=0.145\linewidth]{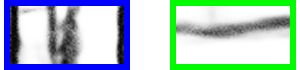}    
    \includegraphics[width=0.145\linewidth]{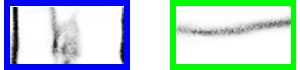}
    \includegraphics[width=0.145\linewidth]{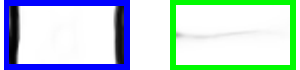}
\\
    
    \includegraphics[width=0.143\linewidth,frame]{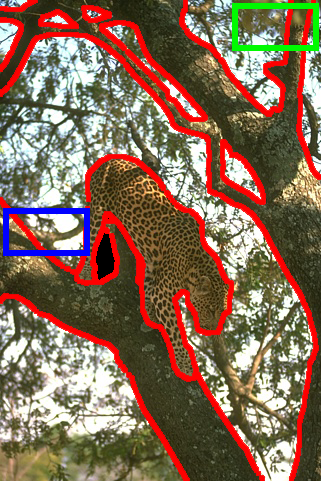}
    \includegraphics[width=0.143\linewidth,frame]{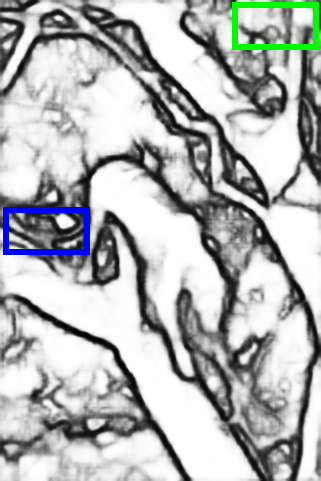}
    \includegraphics[width=0.143\linewidth,frame]{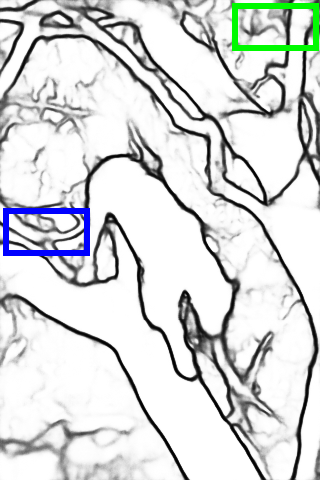}
    \includegraphics[width=0.143\linewidth,frame]{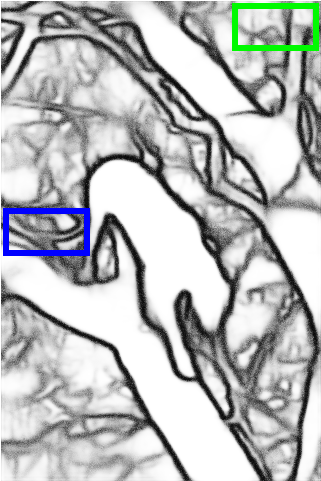}    
    \includegraphics[width=0.143\linewidth,frame]{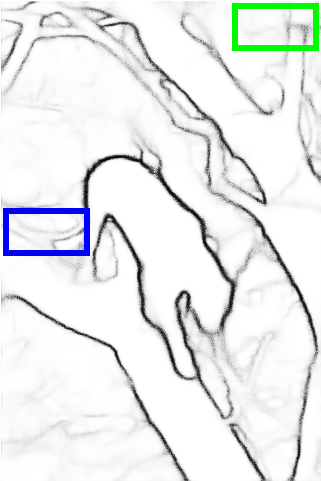}
    \includegraphics[width=0.143\linewidth,frame]{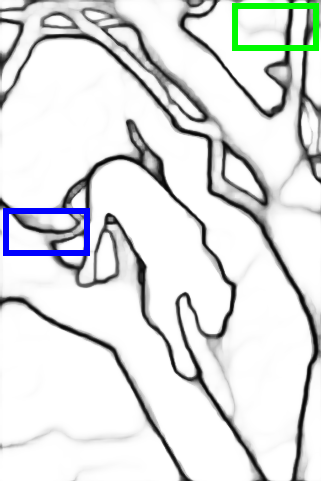}
\\ 

    \includegraphics[width=0.145\linewidth]{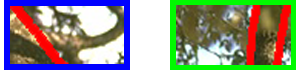}
    \includegraphics[width=0.145\linewidth]{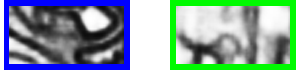}
    \includegraphics[width=0.145\linewidth]{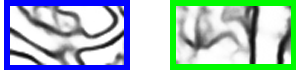}
    \includegraphics[width=0.145\linewidth]{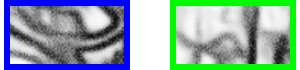}    
    \includegraphics[width=0.145\linewidth]{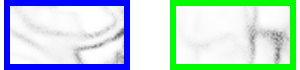}
    \includegraphics[width=0.145\linewidth]{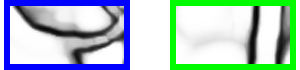}
\\ 

    \includegraphics[width=0.143\linewidth,frame]{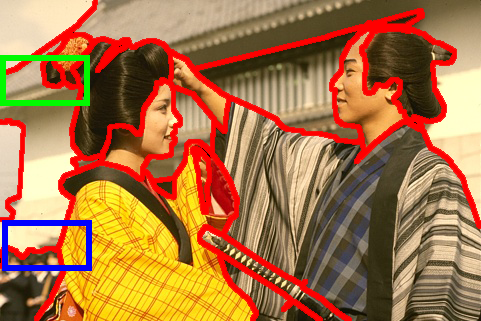}
    \includegraphics[width=0.143\linewidth,frame]{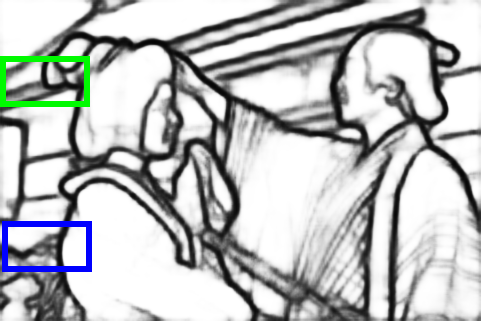}
    \includegraphics[width=0.143\linewidth,frame]{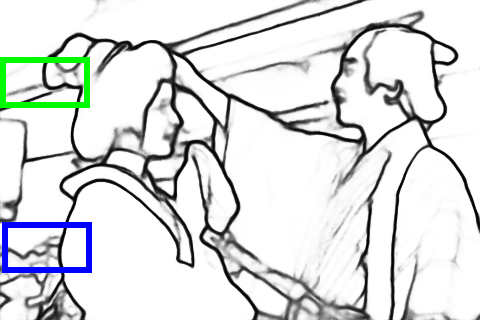}
    \includegraphics[width=0.143\linewidth,frame]{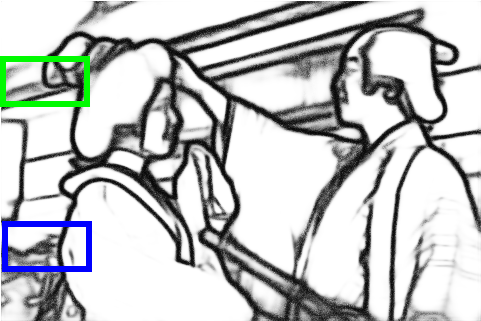}
    \includegraphics[width=0.143\linewidth,frame]{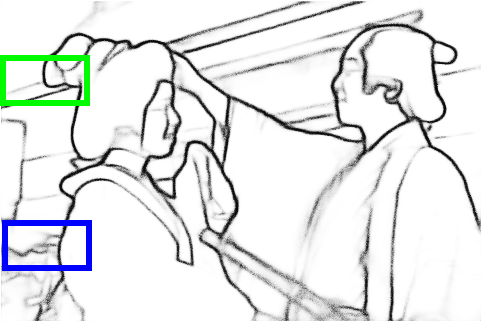}
    \includegraphics[width=0.143\linewidth,frame]{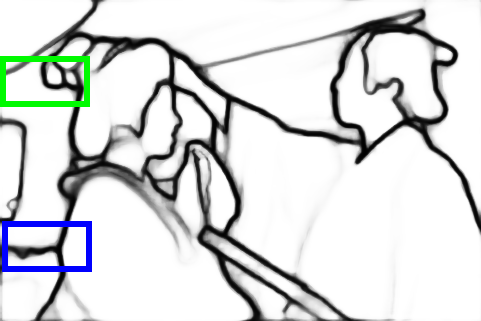}
\\

    \includegraphics[width=0.145\linewidth]{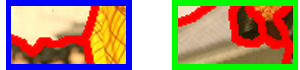}
    \includegraphics[width=0.145\linewidth]{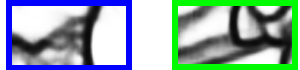}
    \includegraphics[width=0.145\linewidth]{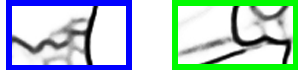}
    \includegraphics[width=0.145\linewidth]{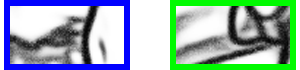}    
    \includegraphics[width=0.145\linewidth]{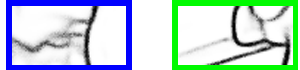}
    \includegraphics[width=0.145\linewidth]{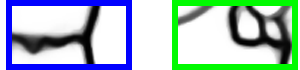}
\\
    
       \; Input\&GT \quad\quad\quad\quad  BDCN \quad\quad\quad\quad\quad EDTER \quad\quad\quad\quad\quad   UAED \quad\quad\quad\quad\quad  MuGE \quad\quad\quad\quad\quad Ours \quad\quad\quad
       
    \caption{Visual comparisons with competitors pre-trained on ImageNet. Our results are obtained by PEdger++ w/ ResNet50.}
    \label{com_pretrain}
\quad

\centering

    \includegraphics[width=0.173\linewidth,frame]{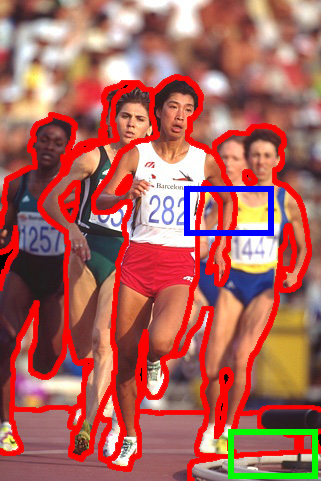}
    \includegraphics[width=0.173\linewidth,frame]{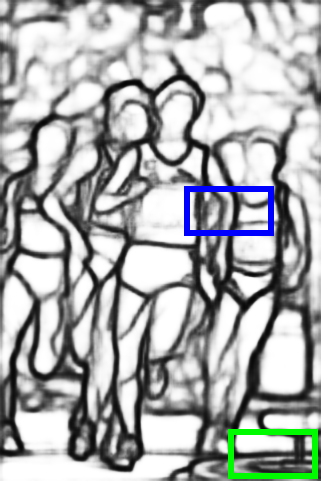}
    \includegraphics[width=0.173\linewidth,frame]{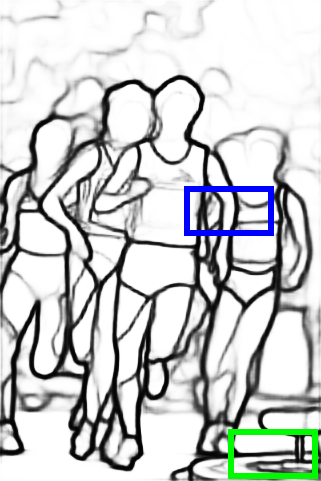}
    \includegraphics[width=0.173\linewidth,frame]{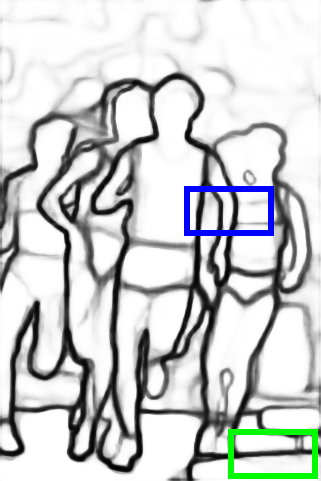}
    \includegraphics[width=0.173\linewidth,frame]{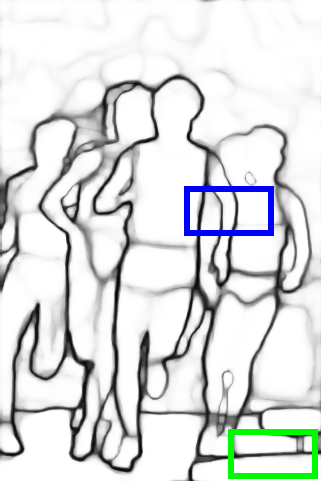}
\\

    \includegraphics[width=0.175\linewidth]{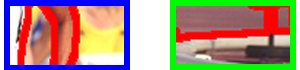}
    \includegraphics[width=0.175\linewidth]{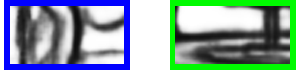}
    \includegraphics[width=0.175\linewidth]{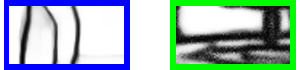}    
    \includegraphics[width=0.175\linewidth]{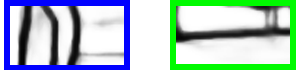}
    \includegraphics[width=0.175\linewidth]{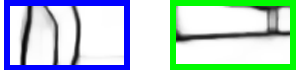}
\\

    \includegraphics[width=0.173\linewidth,frame]{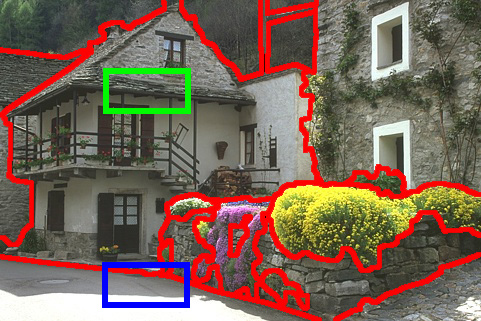}
    \includegraphics[width=0.173\linewidth,frame]{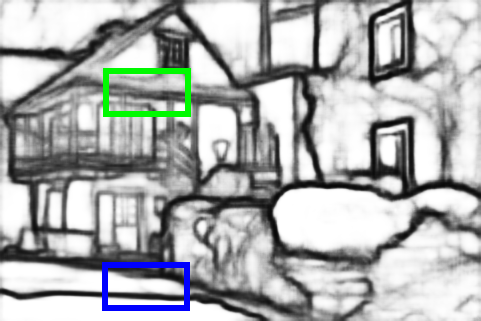}
    \includegraphics[width=0.173\linewidth,frame]{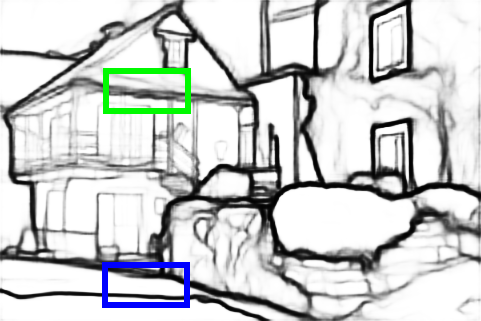}
    \includegraphics[width=0.173\linewidth,frame]{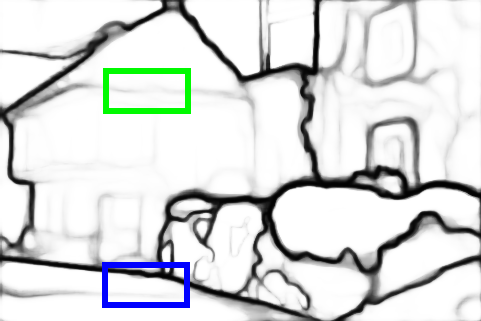}
    \includegraphics[width=0.173\linewidth,frame]{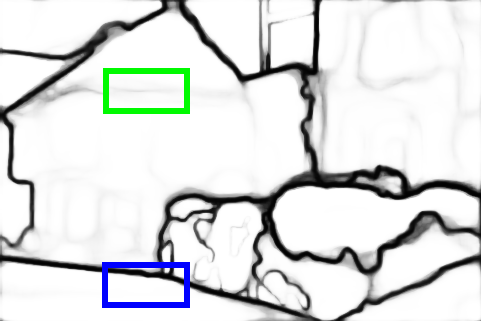}
\\

    \includegraphics[width=0.175\linewidth]{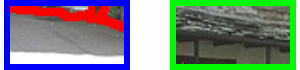}
    \includegraphics[width=0.175\linewidth]{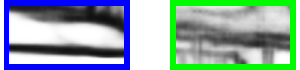}
    \includegraphics[width=0.175\linewidth]{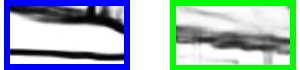}    
    \includegraphics[width=0.175\linewidth]{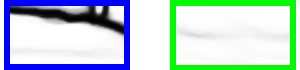}
    \includegraphics[width=0.175\linewidth]{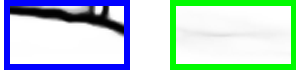}
\\

    \includegraphics[width=0.173\linewidth,frame]{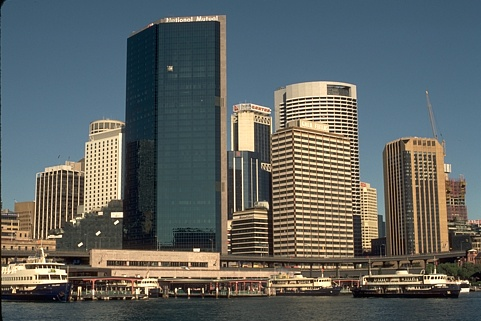}
    \includegraphics[width=0.173\linewidth,frame]{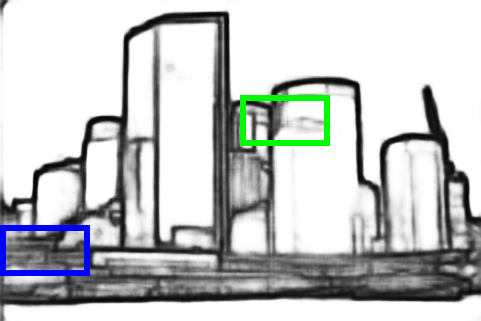}
    \includegraphics[width=0.173\linewidth,frame]{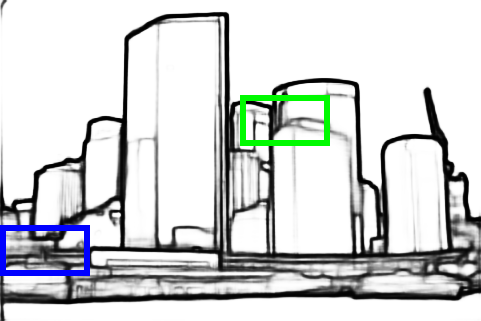}
    \includegraphics[width=0.173\linewidth,frame]{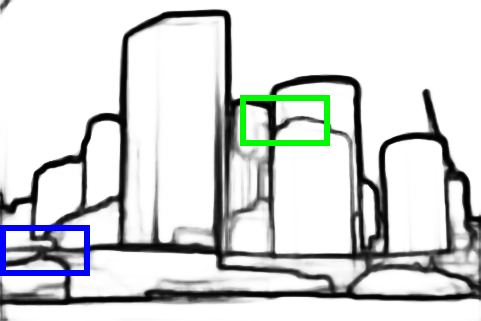}
    \includegraphics[width=0.173\linewidth,frame]{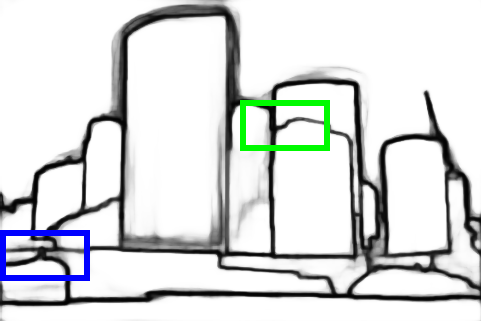}
\\

    \includegraphics[width=0.175\linewidth]{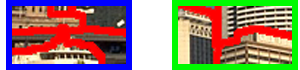}
    \includegraphics[width=0.175\linewidth]{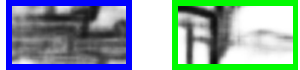}
    \includegraphics[width=0.175\linewidth]{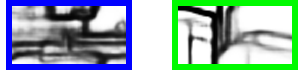}    
    \includegraphics[width=0.175\linewidth]{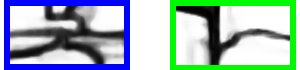}
    \includegraphics[width=0.175\linewidth]{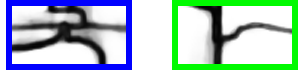}
\\
    
       \;\; Input\&GT \quad\quad\quad\quad\quad\quad PiDiNet \quad\quad\quad\quad\quad\quad   PEdger \quad\quad\quad\quad\quad\quad  PEdger++ \quad\quad\quad PEdger++Large 

    \caption{Visual comparisons with methods in the absence of pre-training.}
    \label{com}
\end{figure*}

\begin{figure}[t]
    \centering
    \subfloat[Input\&GT]{
    \includegraphics[width=0.31\linewidth, frame]{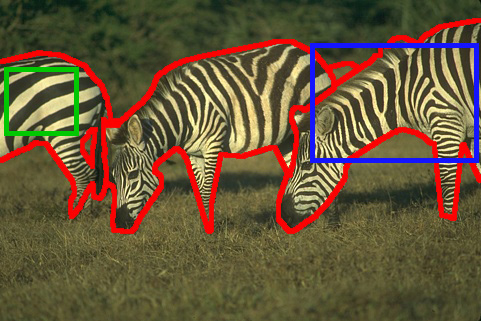}
    }
    \subfloat[Baseline]{
    \includegraphics[width=0.31\linewidth, frame]{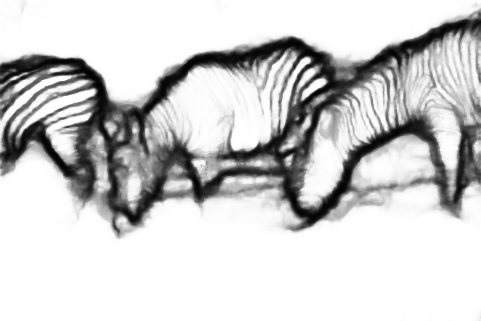}
    }
    \subfloat[NISMP]{
    \includegraphics[width=0.31\linewidth, frame]{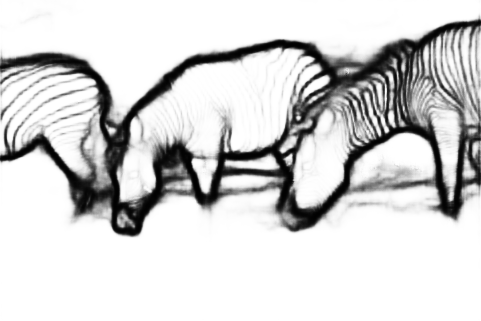}
    }

    \subfloat[EADS]{
    \includegraphics[width=0.31\linewidth, frame]{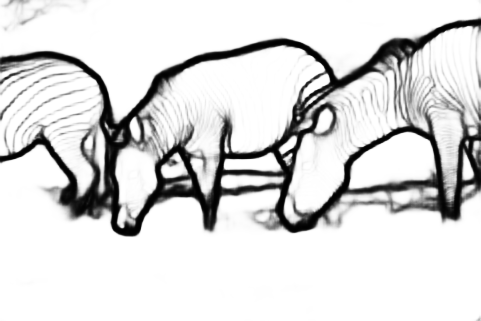}
    }
    \subfloat[EADM]{
    \includegraphics[width=0.31\linewidth, frame]{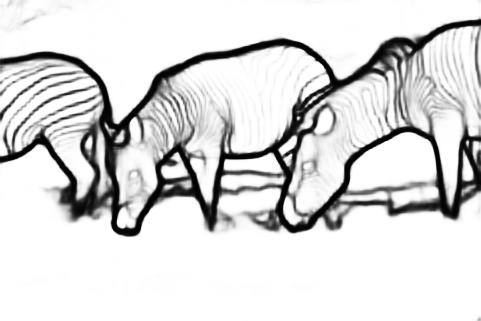}
    }
    \subfloat[EADP]{
    \includegraphics[width=0.31\linewidth, frame]{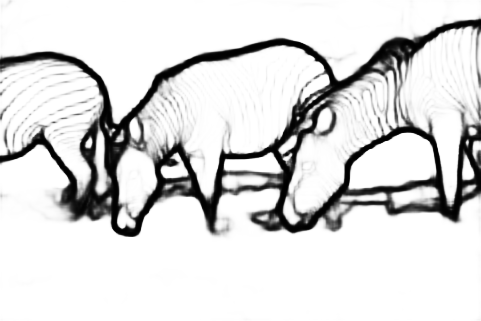}
    }

    \subfloat[ELHS]{
    \includegraphics[width=0.31\linewidth, frame]{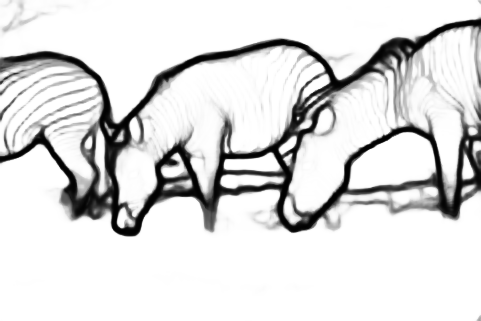}
    }
    \subfloat[Synchron.]{
    \includegraphics[width=0.31\linewidth, frame]{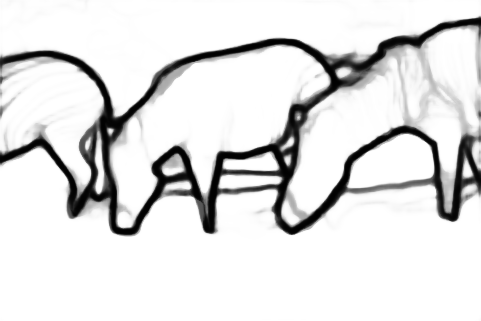}
    }
    \subfloat[Ours]{
    \includegraphics[width=0.31\linewidth, frame]{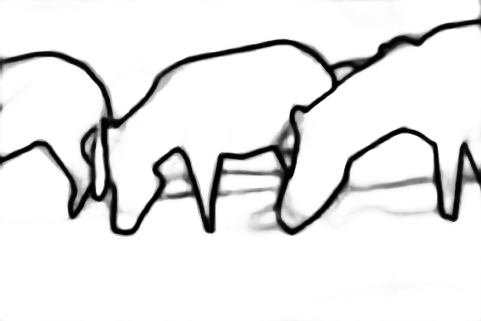}
    }
    
    \caption{Visual comparison between different collaborative learning strategies. Our result is generated by PEdger++ w/ ResNet50, and all the alternatives use pre-trained ResNet50 as backbone.}
    \label{ablation}
\end{figure}

\section{Experimental Validation}
\subsection{Implementation Details}
\label{details}
During training, we employ the SGD optimizer with hyper-parameters configured as follows: the momentum term is set to 0.9, and the weight decay coefficient is 0.001. A warm-up strategy is applied over the first 4 epochs, gradually increasing the learning rate to a peak of 0.001. After the warm-up, the learning rate decreases according to a linear decay schedule. The parameter $\lambda$ in Eq. \eqref{eq:ab} is set to 1.1 for the BSDS and Multicue datasets, and 1.3 for the NYUD. For the recurrent network, $T\defeq 5$ in Eq. \eqref{loss1} representing 5 recurrences, while for the non-recurrent counterpart, $T\defeq 4$ in Eq. \eqref{loss2} corresponding to 4 different scales.  To generate binary edge maps, a threshold of 0.2 is applied to BSDS and Multicue (for binarizing ground truths), while no threshold is needed for NYUD, as images are singly annotated.  The batch size is set to 16, implemented by gradient accumulation due to varying image sizes during training. Data augmentation includes random adjustments to brightness, contrast, saturation, and hue of input images (ranging from 50\% to 150\% of original values), and a 20\% chance of converting RGB to grayscale. During evaluation, standard non-maximum suppression (NMS) \cite{se} is applied to thin detected edges, and both the Optimal Dataset Scale (ODS-F) and Optimal Image Scale (OIS-F) F-scores, together with running speed and model size, are reported. Note that the Sigmoid activation is replaced with $\exp(x-0.50)/(\exp(x-0.5) + \exp(-x+0.5))$ in testing.

\begin{table}[t]
\centering
\caption{Quantitative comparisons among various Bayesian deep learning methods, on the BSDS dataset.}
\begin{tabular}{c|cc}
  \toprule
  Methods & ODS-F & OIS-F  \\
\midrule
\multicolumn{3}{c}{w/o VOC} \\
 \midrule
Robust Bayesian \cite{8186228} & .830 & .845  \\
Non-parametric Ensemble \cite{bayesian4}  & .838 & .849  \\
SUDF \cite{bayesian3} & .833 & .846  \\
Sub-network Inference  \cite{subnetwork} & .839 & .850  \\
SeNeVA  \cite{10657864} & .828 & .840  \\
\midrule
Ours & .841 & .852  \\
 \midrule
\multicolumn{3}{c}{w/ VOC} \\
 \midrule
Robust Bayesian \cite{8186228} & .834 & .844  \\
Non-parametric Ensemble \cite{bayesian4}  & .837 & .846  \\
SUDF \cite{bayesian3} & .831 & .843  \\
Sub-network Inference  \cite{subnetwork} & .841 & .852  \\
SeNeVA \cite{10657864} & .832 & .841  \\
\midrule
Ours & .846 & .856  \\
\bottomrule
\end{tabular}
\label{BayesianCom}
\end{table} 

\begin{table*}[t]
\centering
\caption{Ablation study on the effectiveness of the proposed collaborative learning strategy on the BSDS dataset in terms of ODS-F. All the results are computed with single-scale inputs.}
\begin{tabular}{c|cc|c|cc||c|cc|c|cc}
  \toprule
 \multicolumn{6}{c||}{With Pre-training} & \multicolumn{6}{c}{Without Pre-training} \\
 \midrule
Method & ODS-F & OIS-F & Method & ODS-F & OIS-F & Method & ODS-F & OIS-F & Method & ODS-F & OIS-F  \\
 \midrule
\multicolumn{12}{c}{BSDS w/o VOC} \\
\midrule
  Baseline  & .809 & .820 & EADP & .831 & .842 & Baseline  & .806 & .818 & EADP & .822 & .838 \\
  NISMP  & .819 & .836 &  MLHS & .832 & .844 & NISMP  & .817 & .835 &  MLHS & .820 & .837  \\
  EADS  & .825 & .841 &  Synchron. & .833 & .846 & EADS  & .820 & .838 &  Synchron. & .824 & .840 \\
  EADM  & .830 & .843 &  Ours & .841 & .852  & EADM  & .823 & .840 &  Ours & .830 & .846  \\
\midrule
\multicolumn{12}{c}{BSDS w/ VOC} \\
\midrule
  Baseline  & .820 & .830 & EADP & .835 & .848 & Baseline  & .811 & .824 & EADP & .827 & .843 \\
  NISMP  & .832 & .845 &  MLHS & .837 & .849 & NISMP  & .823 & .839 &  MLHS & .826 & .841  \\
 EADS  & .837 & .851 &  Synchron. & .839 & .850  &  EADS  & .825 & .842 &  Synchron. & .829 & .846 \\
  EADM  & .836 & .851 &  Ours & .846 & .856   & EADM  & .828 & .844 &  Ours & .835 & .850  \\
\bottomrule
\end{tabular}
\label{ablation}
\end{table*}

\begin{table*}[t]
\centering
\caption{Ablation study on the effectiveness of using the validation set during training, on the BSDS dataset in terms of ODS-F. All results are computed with single-scale inputs.} 
\begin{tabular}{c|cc|c|cc||c|cc|c|cc}
  \toprule
\multicolumn{6}{c||}{With Pre-training} & \multicolumn{6}{c}{Without Pre-training} \\
 \midrule
Method & ODS-F & OIS-F & Method & ODS-F & OIS-F & Method & ODS-F & OIS-F & Method & ODS-F & OIS-F \\
 \midrule
\multicolumn{12}{c}{BSDS w/o VOC} \\
\midrule
 Addition  & .829 & .843 & Confidence & .830 & .845 &  Addition  & .821 & .835 & Confidence & .822 & .837 \\
  CrossEntropy  & .827 & .840 &  \textbf{Ours} & \textbf{.841} & \textbf{.852} & CrossEntropy  & .819 & .833 &  \textbf{Ours} & \textbf{.830} & \textbf{.846}  \\
  \midrule
\multicolumn{12}{c}{BSDS w/ VOC} \\
\midrule
  Addition  & .831 & .842 & Confidence & .833 & .846 & Addition  & .828 & .842 & Confidence & .826 & .840 \\
  CrossEntropy  & .835 & .849 &  \textbf{Ours} & \textbf{.846} & \textbf{.856} & CrossEntropy  & .825 & .837 &  \textbf{Ours} & \textbf{.835} & \textbf{.850}  \\
\bottomrule
\end{tabular}
\label{validation}
\end{table*}

\begin{table*}[t]
\centering
\caption{Ablation study on the effectiveness of the two network architectures involved in training, on the BSDS dataset in terms of ODS-F. All results are computed with single-scale input. $\dagger$ denotes the non-recurrent structure whose number of parameters decreases as the resolution of features decreases.}
\begin{tabular}{cc|cc||cc|cc}
  \toprule
\multicolumn{4}{c||}{With Pre-training} & \multicolumn{4}{c}{Without Pre-training} \\
 \midrule
 \multicolumn{2}{c|}{Network 1} & \multicolumn{2}{c||}{Network 2} & \multicolumn{2}{c|}{Network 1} & \multicolumn{2}{c}{Network 2} \\
 \midrule
Architecture & ODS-F & Architecture  & ODS-F & Architecture & ODS-F & Architecture  & ODS-F  \\
 \midrule
\multicolumn{8}{c}{BSDS w/o VOC} \\
\midrule
 Recurrent  & .821 & Recurrent & .824 & Recurrent  & .819 & Recurrent & .821 \\
 Non-Recurrent & .827 & Non-Recurrent & .830 & Non-Recurrent & .822 & Non-Recurrent & .826  \\
 Recurrent & .832 & Non-Recurrent$\dagger$ & .834 & Recurrent & .820 & Non-Recurrent$\dagger$ & .823 \\
 \textbf{Recurrent} & \textbf{.837} & \textbf{Non-Recurrent} & \textbf{.841} &  \textbf{Recurrent} & \textbf{.827} & \textbf{Non-Recurrent} & \textbf{.830}  \\
 \midrule
\multicolumn{8}{c}{BSDS w/ VOC} \\
\midrule
 Recurrent  & .833 & Recurrent & .836 & Recurrent  & .825 & Recurrent & .827  \\
 Non-Recurrent & .837 & Non-Recurrent & .838 & Non-Recurrent & .827 & Non-Recurrent & .830 \\
 Recurrent & .835 & Non-Recurrent$\dagger$ & .839 & Recurrent & .824 & Non-Recurrent$\dagger$ & .828 \\
 \textbf{Recurrent} & \textbf{.840} & \textbf{Non-Recurrent} & \textbf{.846} &  \textbf{Recurrent} & \textbf{.831} & \textbf{Non-Recurrent} & \textbf{.835}\\
\bottomrule
\end{tabular}
\label{netarch}
\end{table*} 

\begin{table*}[t]
\centering
\caption{Parameter study on the BSDS dataset. The best results are highlighted in \textbf{bold}.}
\begin{tabular}{c|cc|c|cc||c|cc|c|cc}
\toprule
\multicolumn{6}{c||}{With Pre-training} & \multicolumn{6}{c}{Without Pre-training} \\
\midrule
Setting & ODS-F & OIS-F & Setting & ODS-F & OIS-F & Setting & ODS-F & OIS-F & Setting & ODS-F & OIS-F \\
\midrule
\multicolumn{12}{c}{BSDS w/o VOC} \\
\midrule
$T \defeq 3$ & .828 & .843 & $\eta_{J} \defeq 0.4$ & .833 & .846 & $T \defeq 3$ & .821 & .837 & $\eta_{J} \defeq 0.4$ & .824 & .837 \\
$T \defeq 4$ & .831 & .846 & $\eta_{J} \defeq 0.6$ & .835 & .849  & $T \defeq 4$ & .823 & .841 & $\eta_{J} \defeq 0.6$ & .829 & .843 \\
$T \defeq 5$ & \textbf{.841} & \textbf{.852} & $\eta_{J} \defeq 0.8$ & \textbf{.841} & \textbf{.852} & $T \defeq 5$ & \textbf{.830} & \textbf{.846} & $\eta_{J} \defeq 0.8$ & \textbf{.830} & \textbf{.846}\\
$T \defeq 6$ & .834 & .848 & $\eta_{J} \defeq 1$ & .838 & .851 & $T \defeq 6$ & .828 & .843 & $\eta_{J} \defeq 1$ & .831 & .846\\
\midrule
$S \defeq 1$ & .830 & .845 & $N_v \defeq 5\% N$ & .831 & .849 & $S \defeq 1$ & .822 & .836 & $N_v \defeq 5\% N$ & .820 & .834 \\
$S \defeq 2$ & .832 &  .850 & $N_v \defeq 10\% N$ & .836 & .851 & $S \defeq 2$ & .826 &  .842 & $N_v \defeq 10\% N$ & .825 & .840 \\
$S \defeq 3$ & \textbf{.841} & \textbf{.852} & $N_v \defeq 20\% N$ & .838 & .850  & $S \defeq 3$ & \textbf{.830} & .846 & $N_v \defeq 20\% N$ & \textbf{.831} & .844 \\
$S \defeq 4$ & .835 & .846 & $N_v \defeq 30\% N$ & \textbf{.841} & \textbf{.852}  & $S \defeq 4$ & .829 & \textbf{.843} & $N_v \defeq 30\% N$ & .830 & \textbf{.846} \\
$S \defeq 5$ & .833 & .841 & $N_v \defeq 40\% N$ & .837 & .848  & $S \defeq 5$ & .824 & .839 & $N_v \defeq 40\% N$ & .826 & .840 \\
\midrule
\multicolumn{12}{c}{BSDS w/ VOC} \\
\midrule
$T \defeq 3$ & .835 & .848 & $\eta_{J} \defeq 0.4$ & .839 & .852 & $T \defeq 3$ & .826 & .840 & $\eta_{J} \defeq 0.4$ & .830 & .843 \\
$T \defeq 4$ & .837 & .851 & $\eta_{J} \defeq 0.6$ & .841 & .853 & $T \defeq 4$ & .831 & .845 & $\eta_{J} \defeq 0.6$ & .834 & .842 \\
$T \defeq 5$ & \textbf{.846} & \textbf{.856} & $\eta_{J} \defeq 0.8$ & \textbf{.846} & \textbf{.856} & $T \defeq 5$ & \textbf{.835} & \textbf{.850} & $\eta_{J} \defeq 0.8$ & \textbf{.835} & \textbf{.850}\\
$T \defeq 6$ & .839 & .847 & $\eta_{J} \defeq 1$ & .843 & .854 & $T \defeq 6$ & .834 & .848 & $\eta_{J} \defeq 1$ & .833 & .846\\
\midrule
$S \defeq 1$ & .834 & .851 & $N_v \defeq 5\% N$ & .833 & .841 & $S \defeq 1$ & .827 & .842 & $N_v \defeq 5\% N$ & .826 & .837 \\
$S \defeq 2$ & .837 &  .854 & $N_v \defeq 10\% N$ & .835 & .848 & $S \defeq 2$ & .830 &  .846 & $N_v \defeq 10\% N$ & .828 & .844  \\
$S \defeq 3$ & \textbf{.846} & \textbf{.856} & $N_v \defeq 20\% N$ & .837 & .850 & $S \defeq 3$ & \textbf{.835} & .850 & $N_v \defeq 20\% N$ & .831 & .845 \\
$S \defeq 4$ & .840 & .853 & $N_v \defeq 30\% N$ & \textbf{.846} & \textbf{.856} & $S \defeq 4$ & .834 & \textbf{.851} & $N_v \defeq 30\% N$ & \textbf{.835} & \textbf{.850} \\
$S \defeq 5$ & .842 & .851 & $N_v \defeq 40\% N$ & .840 & .852 & $S \defeq 5$ & .832 & .847 & $N_v \defeq 40\% N$ & .832 & .846 \\
\bottomrule
\end{tabular}
\label{param}
\end{table*} 

\subsection{Datasets \& Competitors}
Three datasets employed for evaluation include:
1) The \textit{BSDS500} \cite{gpu-ucm} has 200 training, 100 validation and 200 testing images with labels annotated by 4 to 9 participants. Consistent with prior works, we train on two variants of training data: a combination of augmented BSDS and PASCAL VOC Context data, and the augmented BSDS dataset alone; and 2) The \textit{NYU Depth (NYUD)} dataset \cite{rgbd} contains 1449 paired RGB and HHA images, split into 381 training images, 414 validation images, and 654 testing images. Following previous works \cite{HED,EDTER,PiDiNet,PiDiNet_TPAMI,treasure,MuGE}, we adjust the maximum tolerance for correct edge prediction matches with ground truth to 0.011 during evaluation. 
3) The \textit{Multicue} dataset comprises 100 challenging scenes with both edge and boundary annotations. Each scene includes short left- and right-view sequences. Following previous studies \cite{RCF,PiDiNet,PiDiNet_TPAMI,EDTER,treasure,MuGE},  we apply augmentation by rotating images at $90^\circ$, $180^\circ$, and $270^\circ$ angles.

We benchmark our models against several traditional algorithms, including:  ISCRA \cite{ISCRA}, SE \cite{se}, and OEF \cite{oef}, as well as deep learning methods, including  CEDN \cite{CEDN}, HED \cite{HED}, COB \cite{COB}, CED \cite{CED}, AMH-Net \cite{AMH}, LPCB \cite{LPCB}, RCF \cite{RCF}, BDCN \cite{BDCN_cvpr}, DSCD \cite{DSCD}, PiDiNet \cite{PiDiNet_TPAMI}, LDC \cite{LDC}, FCL \cite{FCL},  EDTER \cite{EDTER}, DiffusionEdge \cite{Diffusionedge}, UAED \cite{treasure},  MuGE \cite{MuGE}, and SAUGE \cite{SAUGE}.

\subsection{Comparison with State-of-the-arts}
We compare our method against other state-of-the-art models, including those with and without pre-training, to demonstrate our superiority and showcase its robustness across diverse configurations. 

It can be seen from Tabs. \ref{BSDS_result1}, \ref{BSDS_result2} and \ref{BSDS_results_MS} together with Precision-Recall curves in Fig. \ref{pr} that,  our method achieves high detection accuracy under different computational complexities. Specifically, on the one hand,  in scenarios with pre-training, we utilize pre-trained VGG16 or ResNet50 as the encoding module (blue part in Fig. \ref{arch}) of the non-recurrent network.  To align with our PEdger++, only the first four stages of VGG16 or ResNet50 are adopted as the encoding part, which generates the consistent number of side-outputs. Owing to the difficulty of injecting pre-trained backbones directly into the recurrent structure, we opt to train two non-recurrent structures collaboratively, both of which use identical pre-trained backbones. As illustrated in Fig. \ref{com_pretrainandnopretrain}, compared to PEdger and PEdger++ without any pre-training, the use of pre-trained parameters enables the model to produce edges that are more perceptually meaningful, which underscores the value of pre-training in capturing high-level semantics. When comparing with the methods with pre-training on ImageNet \cite{ImageNet} or SA-1B \cite{sam},  our method reaches the state-of-the-art accuracy, and generates visually pleasing edge maps as illustrated in Tab. \ref{BSDS_result1} and \ref{BSDS_result2}, and Fig. \ref{com_pretrain}. For example, our PEdger++ w/ ResNet50 attains ODS-F and OIS-F of 0.846 and 0.856 when trained on the mixture of augmented BSDS and PASCAL data, respectively, while still maintaining a swift running speed. This performance surpasses not only the recently proposed UAED, but also EDTER, DiffusionEdge, and SAUGE that rely on computationally intensive ViT, diffusion models, and SAM \cite{sam}, respectively. Notice that both EDTER and DiffusionEdge require troublesome two-stage training, making them not so practical. Despite that PEdger++ w/ ResNet50 lags slightly behind MuGE, it is capable of exceeding MuGE with an expansion of the parameter amount to 12.7M (PEdger++Large w/ ResNet50) at a 10x faster running speed than MuGE.

On the other hand,  in scenarios without pre-training, we have a significant advantage in terms of both ODS-F and OIS-F scores,  compared with  PiDiNet and PEdger. Though our accuracy is lower than that of UAED, and MuGE, we can operate at a speed about 80$\times$ faster than them, and requires merely 716K parameters. When increasing the parameter amount to 4.3M, our model reaches higher accuracy than UAED. The visual results are presented in Fig. \ref{com}, wherein it is evident that our results show clean edges, and a reduction in incorrect predictions within the non-edge regions.  

When trained using our efficient version of collaborative learning algorithm as illustrated \textbf{Algorithm} \ref{alg:algorithm2}, the total training time can be greatly shortened. As given in Tab. \ref{rate}, this efficient training algorithm  evidently reduces the training time, with negligible accuracy degradation.  \emph{$\Delta$ ODS-F}, \emph{$\Delta$ OIS-F}, and \emph{$\Delta$ Training Time} in Tab. \ref{rate} are all  calculated through dividing the absolute value of the difference between the results of \textbf{Algorithm} \ref{alg:algorithm} and \textbf{Algorithm} \ref{alg:algorithm2}, by that of \textbf{Algorithm} \ref{alg:algorithm}.

The results on the NYUD and Multicue datasets are presented in Tab. \ref{NYUD} and Tab. \ref{multicue}, respectively. It can be seen that our approach attains satisfactory results that are on par with other state-of-the-art models, while also maintaining computational efficiency. It is worth mentioning that, we use identical network architecture, \textit{i.e.}, the same number of channels and network layers, across the BSDS, NYUD, and Multicue datasets.

\subsection{Network Scalability}
In this part, we perform scalability experiments to investigate the impacts of varying model sizes, which enables adaptations for a broad range of application scenarios. The configuration settings include: Tiny, Small, Normal, and Large. The model sizes are changed through tuning the number of channels and the depth of network layers. Comprehensive quantitative results on the BSDS, NYUD, and Multicue datasets are provided in Tab. \ref{scalability} to assess performance across different computational complexities.
As can be observed, reducing the model size generally leads to a decline in detection accuracy, as measured by both ODS-F and OIS-F scores. The Large settings significantly boosts accuracy even when trained from scratch, outperforming most contemporary deep edge detection models. Notably, the proposed PEdger++ consistently achieves higher accuracies than the previous PEdger. 

\subsection{Discussion on Bayesian Deep Learning}
To showcase the effectiveness of our Bayesian posterior approximation strategy, we compare our quantitative results with other Bayesian deep learning methods, including: 1) \emph{Robust Bayesian} \cite{8186228} that employs a large margin constraint to enhance the robustness in Bayesian learning; 2) \emph{Non-parametric Ensemble} \cite{bayesian4} that probes model parameters at different scales without manually tuning hyper-parameters; 3) \emph{SUDF} \cite{bayesian3} that performs Bayesian inference in both spatial and frequency domains; 4) \emph{Sub-network Inference} \cite{subnetwork}  that regards merely a small subset of parameters as probabilistic while keeping others deterministic; 5) \emph{SeNeVA} \cite{10657864} that assigns the balancing weights to each parameter samplings through a separate assignment network. We re-implement these above competing approaches within our framework to facilitate their direct application to edge detection. For fair comparison, the number of parameter samples $S$, total training epochs $J$, and other hyper-parameters, are consistent across different methods. 
It can be seen from Tab. \ref{BayesianCom} that, our proposed method for Bayesian posterior approximation, particularly when considering generalization to unseen data, exhibits a significant performance improvement over competing methods. This advantage highlights the robustness and reliability of our approach for edge detection, underscoring its potential for broader applications.

\subsection{Ablation Study}
We conduct the following ablation studies all on the mixture of augmented BSDS and PASCAL VOC \cite{PASCAL_VOC} dataset.

\textbf{Effectiveness of Assembling Cross Information.} We evaluate 6 configurations, including 1) \emph{Baseline.} A single non-recurrent model trained with original hard targets $\mathbf{Y}_{n_t}$; 2)  \emph{No Integration from Structures, Moments, and Parameter samplings (NISMP)}. Soft targets $\tilde{\mathbf{Y}}_{n_t}$ are generated solely from a single model trained up to the last epoch, without knowledge integration from different structures, training moments, or parameter samplings; 3) \emph{Ensemble Across Different Training Moments Only (EADM)}. Soft targets $\tilde{\mathbf{Y}}_{n_t}$ are generated using knowledge across training moments only. Here, only the non-recurrent momentum and back-propagation networks are used; 4) \emph{Ensemble Across Different Structures Only (EADS)}. Generating soft targets $\tilde{\mathbf{Y}}_{n_t}$ leverages knowledge from different network structures, without integration across training moments or parameter samplings. The knowledge learned from the last epoch is fused to obtain $\mathbf{M}_{n_t}$; 5) \emph{Ensemble Across Different Parameter Samplings Only (EADP)}. Soft targets $\tilde{\mathbf{Y}}_{n_t}$ are produced using multiple sampled parameters, without integration from structures or training moments. A single non-recurrent model provides the learned knowledge after the last epoch; and 6) \emph{Mutual Learning of Heterogeneous Structures (MLHS)}. The recurrent and non-recurrent network will mutually supervise each other. The recurrent model is supervised by the non-recurrent model and $\mathbf{\mathbf{Y}}_{n_t}$, while the non-recurrent model is supervised by the recurrent one and $\mathbf{\mathbf{Y}}_{n_t}$. In this setting, the knowledge across training moments and parameter samplings, is incorporated.

As shown in Tab. \ref{ablation}, the Baseline model exhibits the lowest accuracy among all alternatives, due to its lack of robust knowledge integration from diverse sources. NISMP, which replaces original targets $\mathbf{Y}_{n_t}$ with soft targets $\tilde{\mathbf{Y}}_{n_t}$, shows a slight improvement over the Baseline but still performs below the other configurations. In contrast, the models trained with targets $\tilde{\mathbf{Y}}_{n_t}$ generated by EADM, EADS, or EADP demonstrate notable performance gains, underscoring the importance of integrating knowledge from heterogeneous architectures, training moments, or multiple parameter samplings to enhance robustness. When combining knowledge from all three aspects (\emph{Ours}), the best results are achieved. Additionally, MLHS shows inferior performance compared to ours, indicating the benefit of fusing the predictions by recurrent and non-recurrent networks via Eq. \eqref{uncertainty}. Please refer to Fig. \ref{ablation} for visual results of ablation studies.

\textbf{Inferiority of Two-stage Training.} To assess the impact of two-stage training, we first maintain a momentum network with the recurrent structure and its corresponding back-propagation network, where $\tilde{\mathbf{Y}}_{n_t}$ is generated based solely on the predictions of the recurrent momentum network. In the second stage, only the momentum and back-propagation networks with the non-recurrent structure are updated, while the recurrent networks remain fixed.  The target $\tilde{\mathbf{Y}}_{n_t}$ used for supervision in the second stage is obtained according to Eqs. \eqref{uncertainty} and \eqref{soft}. The primary difference between this two-stage training setup and our collaborative learning lies in whether the recurrent network parameters are updated throughout. The quantitative results of two-stage training (\emph{Synchron.} in Tab. \ref{ablation}) are lower than those of ours. This outcome may be due to the fixed parameters in the recurrent networks during the second stage, which limits the performance gains achievable in the non-recurrent networks.

\textbf{Effectiveness of Incorporating a Validation Set.} To verify the necessity of introducing the validation set during training, we report the results of three alternatives that do not determine the influence weight according to the generalization ability on the validation set: 1) \emph{Addition}. This approach simply aggregates the predictions from multiple network parameters by summing them directly, without applying ensemble weights;
2) \emph{CrossEntropy}. This version selects the sampled parameters with the minimal cross entropy loss among all $S$ samplings; and
3) \emph{Confidence}. The sampled parameters are determined with respect to the maximal prediction confidence among all $S$ samplings.
All three alternatives described above are trained solely on the original training set $\mathcal{D}$, without partitioning off a validation set. As given in Tab. \ref{validation}, abandoning the validation set during training leads to a noticeable drop in detection accuracy.

\textbf{Effectiveness of Network Architecture Design.} We conduct experiments to examine how the choice of network architectures affects performance. As can be seen in the first two rows of Tab. \ref{netarch}, adopting identical network structures in training, \textit{e.g.}, all the networks having either a recurrent or non-recurrent structure, results in sub-optimal performance. This finding highlights that leveraging two distinct network structures (recurrent and non-recurrent in this study) is helpful for capturing diverse knowledge, which in turn enhances robustness and accuracy. 

Furthermore, as discussed in Section \ref{NetworkArchitecture}, the amount of parameters for the non-recurrent network should increase progressively as feature size decreases. To validate this design, we allocate more parameters to earlier stages with larger feature sizes, and fewer parameters to later stages for comparison. Besides being slower than our original design, this manner also yields lower accuracy for both recurrent and non-recurrent networks, as indicated by  $\dagger$ in Tab. \ref{netarch}. This is because later stages, which are responsible for learning higher-level semantic information, require more parameters than earlier stages, which focus on lower-level details. 

\subsection{Parameter Study}
This part presents the outcomes of tuning the key hyper-parameters and their impacts on performance. We can infer from Tab. \ref{param} that, if the total recurrent step $T$ is low, such as $T \defeq 3$ or $T \defeq 4$, the performance is degraded due to limited receptive fields, which restricts the model's capacity to capture long-range dependencies and contextual information effectively. Increasing $T$ to 6 offers no obvious accuracy gain but raises computational costs. Hence, we set $T \defeq 5$ for all benchmark datasets. We also evaluate the effect of  $\eta_{J}$, finding that optimal performance is achieved when $\eta_{J}\defeq 0.8$. If $\eta_{J}$ is lower, such as $\eta_{J} \defeq 0.4$ or $\eta_{J} \defeq 0.6$, the model utilizes the robust knowledge from diverse sources. Conversely, if $\eta_{J} \defeq 1.0$, the performance is sub-optimal, as the influence of manually annotated labels diminishes, causing the knowledge of $\mathbf{M}_{n_t}$ to be overly relied upon. For the parameter sampling count $S$, we set $S\defeq 3$, by considering the balance between performance gains and computational costs.

Alongside the previously discussed hyper-parameters, we further delve into the impact of the quantity of the validation set, denoted as $N_v$. As can be observed in Tab. \ref{param}, splitting 30\% of the training samples into the validation set yields the best results, while the results of other configurations are sub-optimal. Generally, a scarcity of samples in the validation set compromises the reliability of generalization measurement, whereas an excessive number of samples act as validating data results in the insufficient amount of training data.

\section{Conclusion}
This work presented a novel collaborative learning paradigm for edge detection, termed PEdger++, which tackles a core challenge in deep learning models: achieving high accuracy across varying levels of computational resources. From an ensemble perspective, the proposed edge detection framework integrates knowledge sourced from  heterogeneous network architectures, different training epochs, and multiple parameter samplings to enhance robustness and adaptability. PEdger++ offers flexibility in training, allowing models to be trained from scratch or fine-tuned from ImageNet, thus catering to diverse device capabilities. Extensive experiments have been conducted to demonstrate the effectiveness of PEdger++, outperforming other state-of-the-art methods and making it a compelling solution for edge detection in real-world applications with different computational constraints. This work marks an important advancement in edge detection, and 
potentially guides future research in efficient, scalable deep learning models.

\bibliographystyle{spmpsci}      
\bibliography{reference}   

%
%

\end{document}